\documentclass[12pt]{article}

\usepackage{amsmath}
\usepackage{amssymb}
\usepackage{booktabs}
\usepackage[dvipdfmx]{color}
\usepackage[dvipdfmx]{graphicx}

\usepackage{bm}
\usepackage{amsthm}
\usepackage{url}
\usepackage{subfigure}
\usepackage{algorithm}
\usepackage{algorithmic}
\usepackage{tabulary}
\usepackage[dvipdfmx]{graphicx}

\newtheorem{thm}{Theorem}[section]

\newtheorem{lem}{Lemma}[section]

\numberwithin{equation}{section}

\begin{document}
\makeatletter

\begin{center}
\large{\bf Critical Bach Size Minimizes Stochastic First-Order Oracle Complexity of Deep Learning Optimizer using Hyperparameters Close to One}\\
\small{This work was supported by JSPS KAKENHI Grant Number 21K11773.}
\end{center}\vspace{3mm}

\begin{center}
\textsc{Hideaki Iiduka}\\
Department of Computer Science, 
Meiji University,
1-1-1 Higashimita, Tama-ku, Kawasaki-shi, Kanagawa 214-8571 Japan. 
(iiduka@cs.meiji.ac.jp)
\end{center}

\vspace{2mm}

\footnotesize{
\noindent\begin{minipage}{14cm}
{\bf Abstract:}
Practical results have shown that deep learning optimizers using small constant learning rates, hyperparameters close to one, and large batch sizes can find the model parameters of deep neural networks that minimize the loss functions. 
We first show theoretical evidence that the momentum method (Momentum) and adaptive moment estimation (Adam) perform well in the sense that the upper bound of the theoretical performance measure is small with a small constant learning rate, hyperparameters close to one, and a large batch size. 
Next, we show that there exists a batch size called the critical batch size minimizing the stochastic first-order oracle (SFO) complexity,
which is the stochastic gradient computation cost, and that SFO complexity increases once the batch size exceeds the critical batch size.
Finally, we provide numerical results that support our theoretical results. 
That is, the numerical results indicate that Adam using a small constant learning rate, hyperparameters close to one, and the critical batch size minimizing SFO complexity has faster convergence than Momentum and stochastic gradient descent (SGD).  
\end{minipage}
 \\[5mm]


\hbox to14cm{\hrulefill}\par


\section{Introduction}\label{sec:1}
\subsection{Background}\label{subsec:1.1}
Useful deep learning optimizers have been proposed to find the model parameters of the deep neural networks that minimize loss functions called the expected risk and empirical risk, such as
stochastic gradient descent (SGD) \cite{robb1951,zinkevich2003,nem2009,gha2012,gha2013},
momentum methods \cite{polyak1964,nest1983},
and adaptive methods.
The various adaptive methods include Adaptive Gradient (AdaGrad) \cite{adagrad}, Root Mean Square Propagation (RMSProp) \cite{rmsprop},
Adaptive Moment Estimation (Adam) \cite{adam}, Adaptive Mean Square Gradient (AMSGrad) \cite{reddi2018}, Yogi \cite{NEURIPS2018_90365351},
Adam with decoupled weight decay (AdamW) \cite{loshchilov2018decoupled}, and AdaBelief (named for adapting stepsizes by the belief in observed gradients) \cite{ada}.


Theoretical analyses of adaptive methods for nonconvex optimization were presented in \cite{NEURIPS2018_90365351,cvpr2019,chen2019,opt2020,ada,chen2020_1,iiduka2021} (see \cite{feh2020,chen2020,sca2020,loizou2021} for convergence analyses of SGD for nonconvex optimization).
A particularly interesting feature of adaptive methods is the use of hyperparameters, denoted by $\beta_1$ and $\beta_2$, that can be set to influence the method performance.
Almost all of the previous results for convergence analyses of adaptive methods showed that using $\beta_1$ and $\beta_2$ close to $0$ makes the upper bound of the performance measure $(1/K) \sum_{k=1}^K \mathbb{E}[\|\nabla f(\bm{\theta}_k)\|^2]$ small, where $\nabla f$ is the gradient of a loss function $f \colon \mathbb{R}^d \to \mathbb{R}$, $(\bm{\theta}_k)_{k=1}^K$ is the sequence generated by an optimizer, and $K$ is the number of steps. 

Meanwhile, practical results for adaptive methods were presented in \cite{adam,reddi2018,NEURIPS2018_90365351,cvpr2019,chen2019,ada,chen2020_1}.
These studies have shown that using, for example,  
$\beta_1 \in \{ 0.9, 0.99\}$ and $\beta_2 \in \{ 0.99, 0.999 \}$
provides superior performance for training deep neural networks. 
The practically useful $\beta_1$ and $\beta_2$ are each close to $1$, whereas in contrast, the theoretical results (the above paragraph) show that using $\beta_1$ and $\beta_2$ close to $0$ makes the upper bounds of the performance measures small. 

The practical performance of a deep learning optimizer strongly depends on the batch size.
In \cite{l.2018dont}, it was numerically shown that using an enormous batch size leads to a reduction in the number of parameter updates and model training time. 
The theoretical results in \cite{NEURIPS2018_90365351} showed that using large batch sizes makes the upper bound of $\mathbb{E}[\|\nabla f (\bm{\theta}_k)\|^2]$ of an adaptive method small.  
Convergence analyses of SGD in \cite{chen2020} indicated that
running SGD with a decaying learning rate and large batch size for sufficiently many steps
leads to convergence to a local minimizer of a loss function.
Accordingly, the practical results for large batch sizes match the theoretical ones.

In \cite{shallue2019,zhang2019}, it was studied how increasing the batch size affects the performances of deep learning optimizers. 
In both studies, it was numerically shown that increasing batch size tends to decrease the number of steps $K$ needed for training deep neural networks,  
but with diminishing returns. 
Moreover, it was shown that momentum methods can exploit larger batches than SGD \cite{shallue2019}, 
and that K-FAC and Adam can exploit larger batches than momentum methods \cite{zhang2019}.
Thus, it was shown that momentum and adaptive methods can significantly reduce the number of steps $K$ needed for training deep neural networks \cite[Figure 4]{shallue2019}, \cite[Figure 5]{zhang2019}.

\subsection{Motivation}\label{subsec:1.2}
\subsubsection{Hyperparameters close to one and constant learning rate}
As described in Section \ref{subsec:1.1}, the practically useful $\beta_1$ and $\beta_2$ are each close to $1$, whereas in contrast, the theoretical results show that using $\beta_1$ and $\beta_2$ close to $0$ makes the upper bounds of the performance measures small. Hence, there is a gap between theory ($\beta_1, \beta_2 \approx 0$) and practice ($\beta_1, \beta_2 \approx 1$) for adaptive methods. 
As a consequence, the first motivation of this paper is to bridge this gap. 

Since using small constant learning rates (e.g., $10^{-2}$, $10^{-3}$, and $10^{-4}$) is robust for training deep neural networks
\cite{adam,reddi2018,NEURIPS2018_90365351,cvpr2019,chen2019,ada,chen2020_1}, we focus on using a small constant learning rate $\alpha$.
Here, we note that using a learning rate depending on the Lipschitz constant $L$ of the gradient of a loss function $f$ would be unrealistic. 
This is because computing the Lipschitz constant $L$ is NP-hard \cite[Theorem 2]{NEURIPS2018_d54e99a6}.

\subsubsection{Critical batch size}\label{subsec:1.2.2}
The second motivation of this paper is to clarify theoretically the relationship between the diminishing returns reported in \cite{shallue2019,zhang2019} and batch size.
Numerical evaluations in \cite{shallue2019,zhang2019} have definitively shown that, for deep learning optimizers, the number of steps $K$ needed to train a deep neural network halves for each doubling of the batch size $b$ and that there is a region of diminishing returns beyond the {\em critical batch size} $b^\star$.
This implies that there is a positive number $C$ such that 
\begin{align}\label{fact}
Kb \approx 2^C \text{ for } b \leq b^\star \text{ and }
Kb \geq 2^C \text{ for } b \geq b^\star,
\end{align}
where $K$ and $b$ are defined for $i,j\in\mathbb{N}$ by $K = 2^i$ and $b = 2^j$
(For example, Figure \ref{fig1} in Section \ref{sec:4} shows $C \approx 20$ and $b^\star \approx 2^{11}$ for Adam used to train ResNet-20 on the CIFAR-10 dataset).
We define the {\em stochastic first-order oracle (SFO) complexity} of a deep learning optimizer as $Kb$ on the basis of the number of steps $K$ needed for training the deep neural network and the batch size $b$ used in the optimizer. Let $b^\star$ be a critical batch size such that there are diminishing returns for all batch sizes beyond $b^\star$, as asserted in \cite{shallue2019,zhang2019}. 
This fact, expressed in (\ref{fact}), implies that, while SFO complexity $Kb$ initially almost does not change (i.e., $K$ halves for each doubling of $b$), $Kb$ is minimized at critical batch size $b^\star$, and there are diminishing returns once the batch size exceeds $b^\star$.

\subsection{Contribution}\label{subsec:1.3}
Our results are summarized in Table \ref{table:1}.
Our goal is to find a local minimizer of a loss function $f$ over $\mathbb{R}^d$, i.e., a stationary point $\bm{\theta}^\star \in \mathbb{R}^d$ satisfying 
$\nabla f(\bm{\theta}^\star) = \bm{0}$,
which is equivalent to the variational inequality (VI) defined for all $\bm{\theta} \in \mathbb{R}^d$ by 
$\nabla f(\bm{\theta}^\star)^\top (\bm{\theta}^\star - \bm{\theta}) \leq 0$.
Hence, we use an $\epsilon$-approximation $\mathrm{VI}(K,\bm{\theta}) \leq \epsilon$ (Table \ref{table:1})
as the performance measure of the sequence $(\bm{\theta}_k)_{k\in\mathbb{N}}$ generated by a deep learning optimizer,
where $\epsilon > 0$ is the precision. 

\begin{table}[ht]
\caption{Relationship between batch size $b$ and the number of steps $K$ to achieve an $\epsilon$-approximation of an optimizer using a constant learning rate $\alpha$ and hyperparameters $\beta_1, \beta_2$. 
The critical batch size $b^\star$ minimizes SFO complexity $Kb$
($G$, $\sigma^2$, $M$, and $v_*$ are positive constants, 
$D(\bm{\theta})$ is a positive real number depending on $\bm{\theta} \in \mathbb{R}^d$, and 
$h$ defined by (\ref{h}) is monotone decreasing for $\beta_1$)}\label{table:1}
\centering
\begin{tabular}{c||c|c|c}
\bottomrule
Optimizer & SGD  & Momentum  & Adam     \\
\hline
$C_1$ 
& $\displaystyle{\frac{\mathbb{E}[\|\bm{\theta}_1 - \bm{\theta} \|^2]}{2\alpha}}$      
& $\displaystyle{\frac{\mathbb{E}[\|\bm{\theta}_1 - \bm{\theta} \|^2]}{2\alpha \beta_1}}$      
& $\displaystyle{\frac{d {D}(\bm{\theta}) \sqrt{M}}{2\alpha \beta_1 \sqrt{1-\beta_2}}}$      \\
\hline
$C_2$ 
& $\displaystyle{\frac{\sigma^2 \alpha}{2}}$      
& $\displaystyle{\frac{\sigma^2 \alpha}{2 \beta_1}}$      
& $\displaystyle{\frac{\sigma^2 \alpha}{2 \sqrt{v_*}\beta_1 (1-\beta_1)}}$     \\
\hline
$C_3$ 
& $\displaystyle{\frac{G^2 \alpha}{2}}$      
& $\displaystyle{\frac{G^2 \alpha}{2\beta_1}} + h(\beta_1)$      
& $\displaystyle{\frac{G^2 \alpha}{2\sqrt{v_*}\beta_1(1-\beta_1)}} + h(\beta_1)$     \\
\hline
Upper Bound of VI 
& \multicolumn{3}{c}{$\displaystyle{\mathrm{VI}(K,\bm{\theta}) := \frac{1}{K} \sum_{k=1}^K \mathbb{E}\left[\nabla f(\bm{\theta}_k)^\top (\bm{\theta}_k - \bm{\theta}) \right] \leq \frac{C_1}{K} + \frac{C_2}{b} + C_3 = \epsilon}$}\\
\hline
Steps $K$ and SFO $Kb$ 
& \multicolumn{3}{c}{$\displaystyle{K = \frac{C_1 b}{(\epsilon - C_3)b - C_2}}$ \text{ } $\displaystyle{Kb = \frac{C_1 b^2}{(\epsilon - C_3)b - C_2}}$} \\
\hline
Critical Batch 
& \multicolumn{3}{c}{$\displaystyle{b^\star = \frac{2 C_2}{\epsilon - C_3}}$} \\
\toprule    
\end{tabular}
\end{table}

\subsubsection{Advantage of setting a small constant learning rate and hyperparameters close to one}
We can show that the upper bound $C_1/K + C_2/b + C_3$ of $\mathrm{VI}(K,\bm{\theta})$ becomes small when $\alpha$ is small, $\beta_1$ and $\beta_2$ are close to $1$, and $K$ is large. 
This implies that Momentum and Adam perform well when $\alpha$ is small and $\beta_1$ and $\beta_2$ are each set close to $1$.
Section \ref{subsec:3.1} shows this result in detail.

\subsubsection{Critical batch size}
We can check that the upper bound $C_1/K + C_2/b + C_3$ of $\mathrm{VI}(K,\bm{\theta})$ becomes small when a batch size $b$ is large.
Motivated by the results in \cite{shallue2019,zhang2019} and Section \ref{subsec:1.2.2}, we use SFO complexity $Kb$ as the performance measure of a deep learning optimizer.
We first show that the number of steps $K$ to satisfy $\mathrm{VI}(K,\bm{\theta}) \leq \epsilon$ can be defined as in Table \ref{table:1}.
As a function, $K$ is convex and monotone decreasing with respect to batch size $b$.
Next, we show that SFO complexity $Kb$ defined as in Table \ref{table:1} 
is convex with respect to batch size $b$. 
This result agrees with the fact of (\ref{fact}).
Moreover, SFO complexity $Kb$ is minimized at 
\begin{align}\label{cbs_1}
b^\star = \frac{2 C_2}{\epsilon - C_3}.
\end{align} 
Section \ref{subsec:3.2} shows the above results in detail.
From the above discussion, we conclude that a deep learning optimizer should use the following parameters: 
\begin{itemize}
\item a small constant learning rate $\alpha$ (e.g., $\alpha = 10^{-3}$)
and hyperparameters $\beta_1$ and $\beta_2$ close to $1$ (e.g., $\beta_1, \beta_2 \in \{ 0.9, 0.99, 0.999\}$)
\item the critical batch size $b^\star$ defined by (\ref{cbs_1})
\end{itemize}
The accurate setting of the critical batch size $b^\star$ defined by (\ref{cbs_1}) would be difficult since $b^\star$ in (\ref{cbs_1}) involves unknown parameters, such as $G$ and $\sigma$ (see (\ref{sigma}) and (A1)).
Hence, we would like to estimate appropriate batch sizes using the formula (\ref{cbs_1}) for $b^\star$ before implementing deep learning optimizers. 
Section \ref{sec:4} will discuss estimation of appropriate batch sizes in  detail.

\section{Nonconvex Optimization and Deep Learning Optimizers}\label{sec:2}
This section gives a nonconvex optimization problem in deep neural networks and optimizers for solving the problem under standard assumptions.

\subsection{Nonconvex optimization in deep learning}\label{subsec:2.1}
Let $\mathbb{R}^d$ be a $d$-dimensional Euclidean space with inner product $\langle \bm{x},\bm{y} \rangle := \bm{x}^\top \bm{y}$ inducing the norm $\| \bm{x}\|$ and $\mathbb{N}$ be the set of nonnegative integers. Define $[n] := \{1,2,\ldots,n\}$ for $n \geq 1$. 
Given a parameter $\bm{\theta} \in \mathbb{R}^d$ and a data point $z$ in a data domain $Z$, a machine learning model provides a prediction whose quality is measured by a differentiable nonconvex loss function $\ell(\bm{\theta};z)$. We aim to minimize the expected loss defined for all $\bm{\theta} \in \mathbb{R}^d$ by
\begin{align}\label{expected}
f(\bm{\theta}) = \mathbb{E}_{z \sim \mathcal{D}} 
[\ell(\bm{\theta};z) ]
= \mathbb{E}[ \ell_{\xi} (\bm{\theta}) ],
\end{align}
where $\mathcal{D}$ is a probability distribution over $Z$, $\xi$ denotes a random variable with distribution function $P$, and $\mathbb{E}[\cdot]$ denotes the expectation taken with respect to $\xi$. A particularly interesting example of (\ref{expected}) is the empirical average loss defined for all $\bm{\theta} \in \mathbb{R}^d$ by 
\begin{align}\label{empirical}
f(\bm{\theta}; S) = \frac{1}{n} \sum_{i\in [n]} \ell(\bm{\theta};z_i)
= \frac{1}{n} \sum_{i\in [n]} \ell_i(\bm{\theta}),
\end{align}
where $S = (z_1, z_2, \ldots, z_n)$ denotes the training set and $\ell_i (\cdot) := \ell(\cdot;z_i)$ denotes the loss function corresponding to the $i$-th training data $z_i$. 

We would like to find a stationary point $\bm{\theta}^\star \in \mathbb{R}^d$ of $f$ defined by (\ref{expected}).
A point $\bm{\theta}^\star \in \mathbb{R}^d$ is a stationary point of $f$ 
if and only if $\bm{\theta}^\star \in \mathbb{R}^d$ satisfies 
the following variational inequality (VI): 
for all $\bm{\theta} \in \mathbb{R}^d$, 
$\nabla f(\bm{\theta}^\star)^\top (\bm{\theta}^\star - \bm{\theta}) \leq 0$.
Hence, we use 
\begin{align*}
\mathrm{VI}(K,\bm{\theta}) := \frac{1}{K} \sum_{k=1}^K \mathbb{E}\left[\nabla f(\bm{\theta}_k)^\top (\bm{\theta}_k - \bm{\theta}) \right]
\leq \epsilon
\end{align*}
as the performance measure of the sequence $(\bm{\theta}_k)_{k\in\mathbb{N}}$ generated by a deep learning optimizer,
where $\epsilon > 0$ is the precision and $K$ denotes the number of steps of the optimizer (see also Table \ref{table:1}).

\subsection{Deep learning optimizers}\label{subsec:2.2}
\subsubsection{Conditions}
We assume that a stochastic first-order oracle (SFO) exists such that, for a given $\bm{\theta} \in \mathbb{R}^d$, it returns a stochastic gradient $\mathsf{G}_{\xi}(\bm{\theta})$ of the function $f$ defined by (\ref{expected}), where a random variable $\xi$ is supported on $\Xi$ independently of $\bm{\theta}$. 
The following are standard conditions when considering a deep learning optimizer.
\begin{enumerate}
\item[(C1)] $f \colon \mathbb{R}^d \to \mathbb{R}$ defined by (\ref{expected}) is continuously differentiable.
\item[(C2)] Let $(\bm{\theta}_k)_{k\in \mathbb{N}} \subset \mathbb{R}^d$ be the sequence generated by a deep learning optimizer. For each iteration $k$, 
\begin{align}\label{gradient}
\mathbb{E}_{\xi_k}[ \mathsf{G}_{\xi_k}(\bm{\theta}_k)] = \nabla f(\bm{\theta}_k),
\end{align}
where $\xi_0, \xi_1, \ldots$ are independent samples and the random variable $\xi_k$ is independent of $(\bm{\theta}_l)_{l=0}^k$. There exists a nonnegative constant $\sigma^2$ such that 
\begin{align}\label{sigma}
\mathbb{E}_{\xi_k}\left[ \left\|\mathsf{G}_{\xi_k}(\bm{\theta}_k) - 
\nabla f(\bm{\theta}_k) \right\|^2 \right] \leq \sigma^2.
\end{align}
\item[(C3)] For each iteration $k$, the optimizer samples a batch $B_{k}$ of size $b$ independently of $k$ and estimates the full gradient $\nabla f$ as 
\begin{align*}
\nabla f_{B_k} (\bm{\theta}_k)
:= \frac{1}{b} \sum_{i\in [b]} \mathsf{G}_{\xi_{k,i}}(\bm{\theta}_k),
\end{align*}
where $\xi_{k,i}$ is a random variable generated by the $i$-th sampling in the $k$-th iteration. 
\end{enumerate}
In the case that $f$ is defined by (\ref{empirical}), we have that, for each $k$, $B_k \subset [n]$ and 
\begin{align*}
\nabla f_{B_k} (\bm{\theta}_k)
= \frac{1}{b} \sum_{i \in [b]} \nabla \ell_{\xi_{k,i}} (\bm{\theta}_k).
\end{align*}

\subsubsection{Adam}
Algorithm \ref{algo:1} is the Adam optimizer \cite{adam} under (C1)--(C3).
The symbol $\odot$ in step 6 is defined for all $\bm{x} = (x_i)_{i=1}^d \in \mathbb{R}^d$, $\bm{x} \odot \bm{x} := (x_i^2)_{i=1}^d \in \mathbb{R}^d$,
and $\mathsf{diag}(x_i)$ in step 8 is a diagonal matrix with diagonal components $x_1, x_2, \ldots, x_d$.

\begin{algorithm} 
\caption{Adam \cite{adam}} 
\label{algo:1} 
\begin{algorithmic}[1] 
\REQUIRE
$\alpha \in (0,+\infty)$, 
$b \in (0,+\infty)$,
$\beta_{1} \in (0,1)$, 
$\beta_{2} \in [0,1)$
\STATE
$k \gets 0$, $\bm{\theta}_{0} \in\mathbb{R}^d$, $\bm{m}_{-1} := \bm{0}$, 
$\bm{v}_{-1} := \bm{0}$
\LOOP 
\STATE
$\nabla f_{B_k} (\bm{\theta}_k)
:= b^{-1} \sum_{i\in [b]} \mathsf{G}_{\xi_{k,i}}(\bm{\theta}_k)$
\STATE 
$\bm{m}_k := \beta_{1} \bm{m}_{k-1} + (1-\beta_{1}) \nabla f_{B_k}(\bm{\theta}_k)$
\STATE
$\displaystyle{\hat{\bm{m}}_k := (1-\beta_{1}^{k+1})^{-1}\bm{m}_k}$
\STATE
$\bm{v}_k := \beta_{2} \bm{v}_{k-1} + (1-\beta_{2}) \nabla f_{B_k}(\bm{\theta}_k) \odot \nabla f_{B_k}(\bm{\theta}_k)$
\STATE
$\displaystyle{\hat{\bm{v}}_k := (1-\beta_{2}^{k+1})^{-1}\bm{v}_k}$
\STATE
$\mathsf{H}_k := \mathsf{diag}(\sqrt{\hat{v}_{k,i}})$
\STATE 
$\bm{\theta}_{k+1} := \bm{\theta}_k - \alpha_k \mathsf{H}_k^{-1} \hat{\bm{m}}_k$
\STATE $k \gets k+1$
\ENDLOOP 
\end{algorithmic}
\end{algorithm}

The SGD optimizer under (C1)--(C3) is Algorithm \ref{algo:1} when
$\beta_1 = 0$ and $\mathsf{H}_k$ is the identity matrix.
The Momentum optimizer under (C1)--(C3) is Algorithm \ref{algo:1} when
$\mathsf{H}_k$ is the identity matrix.

\subsubsection{Assumptions}\label{subsec:2.2.3}
We assume the following conditions that were used in \cite[Theorem 4.1]{adam}:
\begin{enumerate}
\item[(A1)] There exist positive numbers $G$ and $B$ such that, for all $k\in \mathbb{N}$, $\| \nabla f (\bm{\theta}_k) \| \leq G$ and $\|\nabla f_{B_k}(\bm{\theta}_k)\| \leq B$.
\item[(A2)] For all $\bm{\theta} \in \mathbb{R}^d$, there exists a positive number $\mathrm{Dist}(\bm{\theta})$ such that, for all $k\in \mathbb{N}$, $\| \bm{\theta}_k - \bm{\theta} \| \leq \mathrm{Dist}(\bm{\theta})$.
\end{enumerate}
Let $(g_{k,i}^2)_{i=1}^d := \nabla f_{B_k}(\bm{\theta}_k) \odot \nabla f_{B_k}(\bm{\theta}_k)$ $(k\in\mathbb{N})$. Assumption (A1) implies that 
$M := \sup \{ \max_{i\in [d]} g_{k,i}^2 \colon k\in\mathbb{N} \} < + \infty$.
Assumption (A2) implies that 
$D(\bm{\theta}) := \sup \{ \max_{i\in [d]} (\theta_{k,i} - \theta_i)^2 \colon k\in\mathbb{N} \} < + \infty$. 
We define $v_* := \inf \{ \min_{i\in [d]} v_{k,i} \colon k \in \mathbb{N} \}$.
Theorem 3 in \cite{reddi2018} shows that there exists a stochastic convex optimization problem such that Adam using 
$\beta_1 < \sqrt{\beta_2}$ (e.g., $\beta_1 = 0.9$ and $\beta_2 = 0.999$) does not converge to the optimal solution.
If 
for all $k\in\mathbb{N}$ and all $i\in [d]$,
$v_{k,i}$ in Adam satisfies
\begin{align}\label{max_1}
\hat{v}_{k+1,i} \geq \hat{v}_{k,i}, 
\end{align}
then Adam (AMSGrad) with a decaying learning rate $\alpha_k = \mathcal{O}(1/\sqrt{k})$ and $\beta_1$ and $\beta_2$ satisfying $\beta_1 < \sqrt{\beta_2}$ can solve the stochastic convex optimization problem \cite[(2), Theorem 4]{reddi2018}. 
We thus assume condition (\ref{max_1}) for Adam to guarantee the convergence of Adam.

\section{Our Results}\label{sec:3}
This section states our theoretical results (Theorem \ref{thm:1}) in Table \ref{table:1} and our contribution (Sections \ref{subsec:3.1} and \ref{subsec:3.2}) in detail.
The proof of Theorem \ref{thm:1} is given in Appendix \ref{appendix:1}.

\begin{thm}\label{thm:1}
The sequence $(\bm{\theta}_k)_{k\in\mathbb{N}}$ generated by
each of SGD, Momentum, and Adam with (\ref{max_1}) under (C1)--(C3) and (A1) and (A2) satisfies the following:

{\em (i)} {\em [Upper bound of $\mathrm{VI}(K,\bm{\theta})$]}
For all $K \geq 1$ and all $\bm{\theta} \in \mathbb{R}^d$,
\begin{align*}
\mathrm{VI}(K,\bm{\theta})
:= 
\frac{1}{K} \sum_{k=1}^K \mathbb{E}\left[ \nabla f(\bm{\theta}_k)^\top (\bm{\theta}_k - \bm{\theta}) \right] 
\leq 
\frac{C_1}{K} 
+ 
\frac{C_2}{b}
+ 
C_3,
\end{align*}
where $C_i$ $(i=1,2,3)$ for SGD are 
\begin{align*}
&C_1 := \frac{\mathbb{E}[\|\bm{\theta}_1 - \bm{\theta} \|^2]}{2 \alpha}, \text{ }
C_2 := \frac{\sigma^2 \alpha}{2}, \text{ }
C_3 := \frac{G^2 \alpha}{2},
\end{align*}
$C_i$ $(i=1,2,3)$ for Momentum are
\begin{align*}
&C_1 := \frac{\mathbb{E}[\|\bm{\theta}_1 - \bm{\theta} \|^2]}{2 \alpha \beta_1}, \text{ }
C_2 := \frac{\sigma^2 \alpha}{2 \beta_1},\\
&C_3 := \frac{G^2 \alpha}{2\beta_1} 
+ \mathrm{Dist}(\bm{\theta})\left\{ \frac{G(1-\beta_1)}{\beta_1} 
+ 2\sqrt{\sigma^2 + G^2}
\left(\frac{1}{\beta_1} + 2 (1-\beta_1) \right)  \right\},
\end{align*}
$C_i$ $(i=1,2,3)$ for Adam with (\ref{max_1}) are
\begin{align*}
&C_1 := \frac{d D(\bm{\theta}) \sqrt{M}}{2 \alpha \beta_1 \sqrt{1-\beta_2}}, \text{ }
C_2 := \frac{\sigma^2 \alpha}{2 \sqrt{v_*} \beta_1(1-\beta_1)},\\
&C_3 := \frac{G^2 \alpha}{2\sqrt{v_*} \beta_1(1-\beta_1)} 
+ \mathrm{Dist}(\bm{\theta})\left\{ \frac{G(1-\beta_1)}{\beta_1} 
+ 2\sqrt{\sigma^2 + G^2}
\left(\frac{1}{\beta_1} + 2 (1-\beta_1) \right)  \right\},
\end{align*}
and the parameters are defined as in Section \ref{subsec:2.2}.

{\em (ii)} {\em [Steps to satisfy $\mathrm{VI}(K,\bm{\theta}) \leq \epsilon$]}
The number of steps $K$ defined by 
\begin{align}\label{K}
K (b) = \frac{C_1 b}{(\epsilon - C_3)b - C_2}
\end{align}
satisfies $\mathrm{VI}(K,\bm{\theta}) \leq \epsilon$
and the function $K(b)$ defined by (\ref{K}) is convex and monotone decreasing with respect to batch size $b$
$(> C_2/(\epsilon - C_3) > 0)$. 

{\em (iii)} {\em [Minimization of SFO complexity]}
The SFO complexity defined by 
\begin{align}\label{Kb}
K (b) b = \frac{C_1 b^2}{(\epsilon - C_3)b - C_2}
\end{align}
is convex with respect to batch size $b$
$(> C_2/(\epsilon - C_3) > 0)$.
The batch size 
\begin{align}\label{cbs}
b^\star := \frac{2 C_2}{\epsilon - C_3}
\end{align}
attains the minimum value
\begin{align*}
K (b^\star) b^\star = \frac{2 C_1 C_2}{(\epsilon - C_3)^2}
\quad\text{of } K(b)b.
\end{align*}
\end{thm}

\subsection{Advantage of setting a small constant learning rate and hyperparameters close to one}\label{subsec:3.1}
We first show theoretical evidence that Momentum and Adam using a small constant learning rate $\alpha$, $\beta_1$ and $\beta_2$ close to $1$, and a large number of steps $K$ perform well.
Theorem \ref{thm:1}(i) indicates that the upper bound of $\mathrm{VI}(K,\bm{\theta})$ for Momentum is 
\begin{align}\label{upper_momentum}
&\mathrm{VI}(K,\bm{\theta}) \leq 
\frac{\mathbb{E}[\|\bm{\theta}_1 - \bm{\theta} \|^2]}{2\alpha \beta_1 K}
+ \frac{\sigma^2 \alpha}{2 \beta_1 b}
+ \frac{G^2 \alpha}{2\beta_1} + h(\beta_1),
\end{align} 
where $\beta_1 \in (0,1)$ and 
\begin{align}\label{h}
h(\beta_1) := \mathrm{Dist}(\bm{\theta})\left\{ \frac{G(1-\beta_1)}{\beta_1} 
+ 2\sqrt{\sigma^2 + G^2}
\left(\frac{1}{\beta_1} + 2 (1-\beta_1) \right)  \right\}.
\end{align}
Since the function $h(\beta_1)$  defined by (\ref{h}) is monotone decreasing, using $\beta_1$ close to $1$ makes $h(\beta_1)$ small. 
Hence, using $\beta_1$ close to $1$ makes the upper bound in (\ref{upper_momentum}) small. 
Moreover, using a small learning rate $\alpha$ makes 
$\sigma^2 \alpha/(2 \beta_1 b)
+ G^2 \alpha/(2\beta_1)$ in (\ref{upper_momentum})
small.
Meanwhile, using a small learning rate $\alpha$ makes $\mathbb{E}[\|\bm{\theta}_1 - \bm{\theta} \|^2]/(2\alpha \beta_1 K)$ in (\ref{upper_momentum}) large. 
Hence, we need to use a large $K$ to make $\mathbb{E}[\|\bm{\theta}_1 - \bm{\theta} \|^2]/(2\alpha \beta_1 K)$ in (\ref{upper_momentum}) small. 
Theorem \ref{thm:1}(i) indicates that the upper bound of $\mathrm{VI}(K,\bm{\theta})$ for Adam is
\begin{align}
\mathrm{VI}(K,\bm{\theta}) 
&\leq
\frac{d {D}(\bm{\theta}) \sqrt{M}}{2\alpha \beta_1 \sqrt{1-\beta_2} K}
+ 
\frac{\alpha (\sigma^2 b^{-1} + G^2)}{2\sqrt{v_*}\beta_1(1-\beta_1) K}
\sum_{k=1}^K \sqrt{1-\beta_2^{k+1}}
+ h(\beta_1) \label{upper_adam_1}\\
&\leq
\frac{d {D}(\bm{\theta}) \sqrt{M}}{2\alpha \beta_1 \sqrt{1-\beta_2} K}
+ 
\frac{\sigma^2 \alpha }{2\sqrt{v_*}\beta_1(1-\beta_1) b}
+
\frac{G^2 \alpha}{2\sqrt{v_*}\beta_1(1-\beta_1)}
+ h(\beta_1), \label{upper_adam_2}
\end{align}
where $\beta_1 \in (0,1)$ and $\beta_2 \in [0,1)$
(while Theorem \ref{thm:1}(i) shows (\ref{upper_adam_2}),
the strict evaluation (\ref{upper_adam_1}) comes from (\ref{B_k_1}) in Appendix \ref{appendix:1}).
Using $\beta_1$ close to $1$ makes $h(\beta_1)$ small. 
Since $1/(\beta_1(1-\beta_1))$ is monotone increasing for $\beta_1 \geq 1/2$, using $\beta_1$ close to $1$ makes $1/(\beta_1(1-\beta_1))$ large.
Hence, we need to set a small $\alpha$ to make 
$\alpha (\sigma^2/b + G^2)/(2\sqrt{v_*}\beta_1(1-\beta_1))$ small.
The function $\sqrt{1-\beta_2^{k+1}}$ is monotone decreasing with respect to $\beta_2$, while using $\beta_2$ close to $1$ makes $1/\sqrt{1-\beta_2}$ large.
When $\beta_2$ close to $1$ is used, we need to use a large $K$ to make 
$d {D}(\bm{\theta}) \sqrt{M}/(2\alpha \beta_1 \sqrt{1-\beta_2} K)$ small.

\subsection{Critical batch size}\label{subsec:3.2}
Here, we can see that the upper bound of $\mathrm{VI}(K,\bm{\theta})$ for SGD is 
\begin{align}\label{upper_sgd}
&\mathrm{VI}(K,\bm{\theta}) \leq 
\frac{\mathbb{E}[\|\bm{\theta}_1 - \bm{\theta} \|^2]}{2\alpha K}
+ \frac{\sigma^2 \alpha}{2 b}
+ \frac{G^2 \alpha}{2}.
\end{align}
Compared with (\ref{upper_momentum}) and (\ref{upper_adam_2}), it seems that, for a fixed small constant learning rate $\alpha$ (e.g., $\alpha = 10^{-3}$), SGD performs better than Momentum and Adam in terms of the upper bound of $\mathrm{VI}(K,\bm{\theta})$.
Meanwhile, we can check that the upper bounds of $\mathrm{VI}(K,\bm{\theta})$ in (\ref{upper_momentum}), (\ref{upper_adam_2}), and (\ref{upper_sgd}) are small when batch size $b$ is large.
Accordingly, we need to check carefully how to set batch sizes to improve theoretical and practical performances of deep learning optimizers.
Motivated by the results in \cite{shallue2019,zhang2019} and Section \ref{subsec:1.2.2}, we use SFO complexity $Kb$ as the performance measure of a deep learning optimizer.

Theorem \ref{thm:1}(ii) indicates that the number of steps $K$ to satisfy $\mathrm{VI}(K,\bm{\theta}) \leq \epsilon$ can be expressed as (\ref{K}).
The function $K(b)$ defined by (\ref{K}) is convex and monotone decreasing.
Hence, the form of $K$ defined by (\ref{K}) supports theoretically the relationship between $K$ and $b$ shown in \cite{shallue2019,zhang2019} (see also Figures \ref{fig1} and \ref{fig3} in this paper).
Theorem \ref{thm:1}(iii) indicates that SFO complexity defined by 
(\ref{Kb})
is convex with respect to batch size $b$. 
This result agrees with the fact of (\ref{fact}) (see also Figures \ref{fig2} and \ref{fig4} in this paper).
Moreover, SFO complexity $Kb$ is minimized at $b^\star$ defined by (\ref{cbs}); e.g., $b_{\mathrm{S}}^\star$ for SGD, 
$b_{\mathrm{M}}^\star$ for Momentum, and $b_{\mathrm{A}}^\star$ for Adam
are respectively  
\begin{align}\label{cbs_s_m_a}
b_{\mathrm{S}}^\star = \frac{\sigma_{\mathrm{S}}^2 \alpha}{\epsilon - C_{3,\mathrm{S}}},
\text{ }
b_{\mathrm{M}}^\star = \frac{\sigma_{\mathrm{M}}^2 \alpha}{\beta_1(\epsilon - C_{3,\mathrm{M}})},
\text{ }  
b_{\mathrm{A}}^\star = \frac{\sigma_{\mathrm{A}}^2 \alpha}{\sqrt{v_*}\beta_1(1-\beta_1)(\epsilon - C_{3,\mathrm{A}})},
\end{align}
where $\sigma_{\mathrm{S}}$, $\sigma_{\mathrm{M}}$, and $\sigma_{\mathrm{A}}$ are positive constants depending on SGD, Momentum, and Adam and $C_3 = C_{3,\mathrm{S}},C_{3,\mathrm{M}}, C_{3,\mathrm{A}}$ are positive constants defined as in Theorem \ref{thm:1}.
From (\ref{cbs_s_m_a}), the lower bounds $b^*$ of 
$b_{\mathrm{S}}^\star$, 
$b_{\mathrm{M}}^\star$, and $b_{\mathrm{A}}^\star$ 
are respectively 
\begin{align*}
b_{\mathrm{S}}^\star > b_{\mathrm{S}}^* := \frac{\sigma_{\mathrm{S}}^2 \alpha}{\epsilon},
\text{ }
b_{\mathrm{M}}^\star > b_{\mathrm{M}}^* := \frac{\sigma_{\mathrm{M}}^2 \alpha}{\beta_1 \epsilon},
\text{ }  
b_{\mathrm{A}}^\star > b_{\mathrm{A}}^* := \frac{\sigma_{\mathrm{A}}^2 \alpha}{\sqrt{v_*}\beta_1(1-\beta_1) \epsilon }.
\end{align*}
If $\sigma_{\mathrm{S}} = \sigma_{\mathrm{M}} = \sigma_{\mathrm{A}}$, $\alpha$ and $\beta_1$ are fixed
(e.g., $\alpha = 10^{-3}$ and $\beta_1 = 0.9$),
and $\sqrt{v_*} (1-\beta_1) < 1$ (e.g., $0.1 \sqrt{v_*} < 1$ when $\beta_1 = 0.9$), 
then
\begin{align}\label{cbs_s_m_a_1}
b_{\mathrm{S}}^* := \frac{\sigma_{\mathrm{S}}^2 \alpha}{\epsilon}
< 
b_{\mathrm{M}}^* := \frac{\sigma_{\mathrm{M}}^2 \alpha}{\beta_1 \epsilon}
< 
b_{\mathrm{A}}^* := \frac{\sigma_{\mathrm{A}}^2 \alpha}{\sqrt{v_*}\beta_1(1-\beta_1) \epsilon }.
\end{align}
Hence, $b_{\mathrm{A}}^*$ is larger than $b_{\mathrm{S}}^*$ and 
$b_{\mathrm{M}}^*$.
Section \ref{sec:4} discusses estimation of appropriate batch sizes in detail.

\section{Numerical Results}\label{sec:4}
We evaluated the performances of SGD, Momentum, and Adam with different batch sizes. 
The metrics were the number of steps $K$ and the SFO complexity $Kb$ satisfying $f(\bm{\theta}_K) \leq 10^{-1}$, where $\bm{\theta}_K$ is generated for each of SGD, Momentum, and Adam using batch size $b$. 
The stopping condition was $200$ epochs. 
The experimental environment consisted of two Intel(R) Xeon(R) Gold 6148 2.4-GHz CPUs with 20 cores each, a 16-GB NVIDIA Tesla V100 900-Gbps GPU, and the Red Hat Enterprise Linux 7.6 OS. 
The code was written in Python 3.8.2 using the NumPy 1.17.3 and PyTorch 1.3.0 packages. 
A constant learning rate $\alpha = 10^{-3}$ was commonly used. 
Momentum used $\beta_1 = 0.9$. 
Adam used $\beta_1 = 0.9$ and $\beta_2 = 0.999$ \cite{adam}.

Here, we use (\ref{cbs_s_m_a_1}) and estimate appropriate batch sizes in the sense that SFO complexity is minimized. 
The definitions of $v_*$ and $v_{k,i}$ (see also (\ref{v_M}) in Appendix \ref{appendix:1}) imply that, for $k\in [K]$ and all $i\in [d]$, 
$v_* \leq v_{k,i} \leq \max_{k\in [K]} \max_{i\in [d]} g_{k,i}^2 =: g_{k^*,i^*}^2
\leq \sum_{i=1}^d g_{k^*,i}^2 = \|\nabla f_{B_{k^*}}(\bm{\theta}_{k^*})\|^2$.
Condition (C2)(\ref{gradient}) implies that 
$\mathbb{E}[\|\nabla f_{B_k}(\bm{\theta}_k)\|^2] \leq \sigma^2/b + \mathbb{E}[\|\nabla f(\bm{\theta}_k)\|^2]$ (see also (\ref{equation_1}) in Appendix \ref{appendix:1}).
Conditions (C2)(\ref{sigma}) and (C3) imply that, if $b$ is large, then $\sigma$ is small. 
Hence, assuming that $\sigma^2/b \approx 0$ implies that
$v_* \leq \|\nabla f(\bm{\theta}_{k^*})\|^2 = \|(1/n) \sum_{i\in[n]} \nabla \ell_i (\bm{\theta}_{k^*})\|^2 = (1/n^2) \|\sum_{i\in[n]} \nabla \ell_i (\bm{\theta}_{k^*})\|^2 =: (1/n^2) \|\bm{G}_{k^*}\|^2 = (1/n^2) \sum_{i\in [d]} G_{k^*,i}^2 \leq (d/n^2) \max_{i\in[d]} G_{k^*,i}^2$. 
Since deep learning optimizers can approximate stationary points of $f$, 
we assume that, for example, $G_{k^*,i} \approx \epsilon$.
Then, (\ref{cbs_s_m_a_1}) implies that 
\begin{align}\label{cbs_s_m_a_2}
b_{\mathrm{S}}^* := \frac{\sigma_{\mathrm{S}}^2}{10^{3} \epsilon},
\text{ } 
b_{\mathrm{M}}^* := \frac{\sigma_{\mathrm{M}}^2}{9\cdot 10^{2} \epsilon},
\text{ } 
b_{\mathrm{A}}^* := \frac{\sigma_{\mathrm{A}}^2}{9\cdot10 \sqrt{v_*} \epsilon} > 
\frac{\sigma_{\mathrm{A}}^2 n}{9\cdot10 \sqrt{d} \epsilon^2}.
\end{align}
Conditions (C2)(\ref{sigma}) and (C3) imply that, if $b$ is large, then $\sigma$ is small. 
If SGD and Momentum can exploit large batch sizes, 
then $\sigma_{\mathrm{S}}$ and $\sigma_{\mathrm{M}}$ are small. 
However, (\ref{cbs_s_m_a_2}) implies that $b_{\mathrm{S}}^*$ and $b_{\mathrm{M}}^*$ must be small when $\sigma_{\mathrm{S}}$ and $\sigma_{\mathrm{M}}$ are small.
Accordingly, SGD and Momentum would not be able to use large batch sizes.
Meanwhile, from (\ref{cbs_s_m_a_2}), we expect that Adam can exploit a large batch size $b_{\mathrm{A}}^\star > b_{\mathrm{A}}^*$, for example, when $b_{\mathrm{A}}^* < 2^{10}$ (CIFAR-10; $b_{\mathrm{A}}^* \approx 1010$, $d = 3072$, $n = 50000$) and 
$2^{11} < b_{\mathrm{A}}^* < 2^{12}$ 
(MNIST; $b_{\mathrm{A}}^* \approx 2380$, $d = 784$, $n = 60000$), where $\sigma_{\mathrm{A}}^2 \approx 10^{-2}$ and $\epsilon \approx 10^{-2}$ are used.

\subsection{ResNet-20 on the CIFAR-10 dataset}
\begin{figure*}[htbp]
\begin{tabular}{cc}
\begin{minipage}[t]{0.45\hsize}
\centering
\includegraphics[width=1\textwidth]{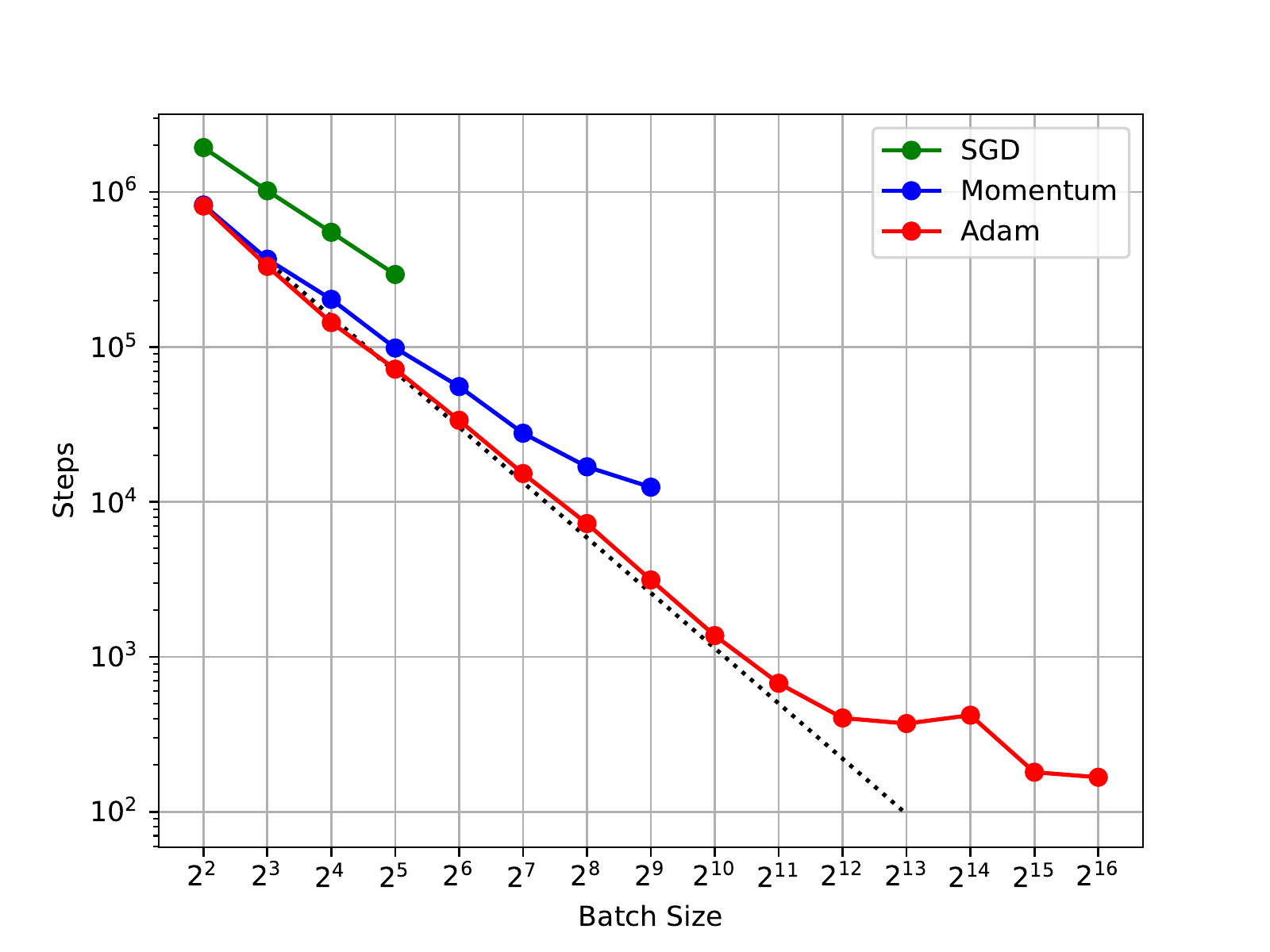}
\caption{Number of steps for SGD, Momentum, and Adam versus batch size needed to train ResNet-20 on CIFAR-10. There is an initial period of perfect scaling (indicated by dashed line) such that the number of steps $K$ for Adam is inversely proportional to batch size $b$. 
Adam has critical batch size $b^\star = 2^{11}$.}
\label{fig1}
\end{minipage} &
\begin{minipage}[t]{0.45\hsize}
\centering
\includegraphics[width=1\textwidth]{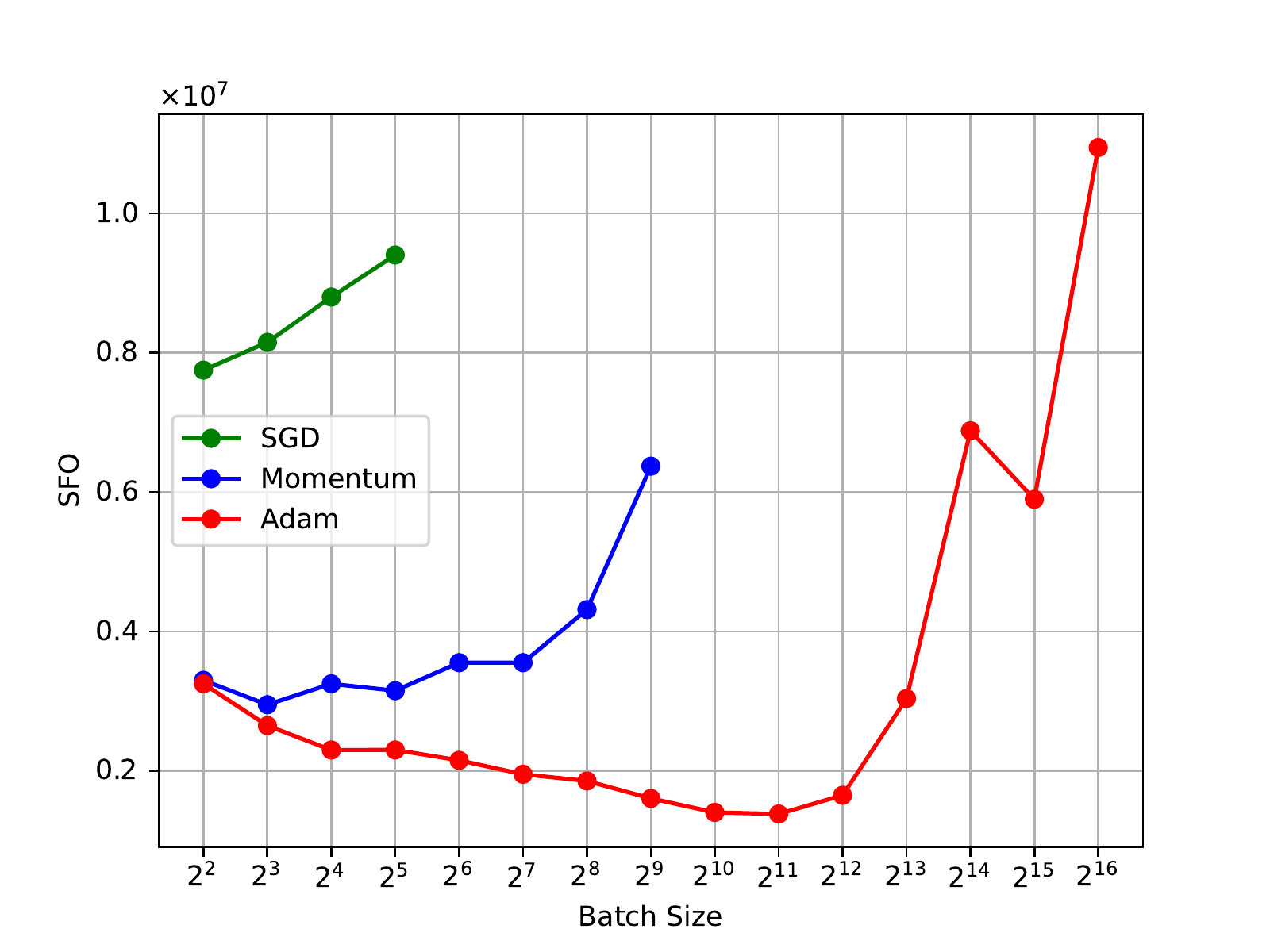}
\caption{SFO complexities for SGD, Momentum, and Adam versus batch size needed to train ResNet-20 on CIFAR-10. 
SFO complexity of Adam (resp. Momentum) is minimized at critical batch size $b^\star = 2^{11}$ (resp. $b^\star = 2^3$), whereas SFO complexity for SGD tends to increase with batch size.}
\label{fig2}
\end{minipage}
\end{tabular}
\end{figure*}

Let us consider training ResNet-20 on the CIFAR-10 dataset with $n = 50000$.
Figure \ref{fig1} shows the number of steps for SGD, Momentum, and Adam versus batch size. For SGD and Momentum, the number of steps $K$ needed for $f(\bm{\theta}_K) \leq 10^{-1}$ initially decreased. However, SGD with $b \geq 2^6$ and Momentum with $b \geq 2^{10}$ did not satisfy $f(\bm{\theta}_K) \leq 10^{-1}$ before the stopping condition was reached. Adam had an initial period of perfect scaling (indicated by dashed line) such that the number of steps $K$ needed for $f(\bm{\theta}_K) \leq 10^{-1}$ was inversely proportional to batch size $b$, and critical batch size $b^\star = 2^{11}$ such that $K$ was not inversely proportional to the batch size beyond $b^\star$; i.e., there were diminishing returns. 

Figure \ref{fig2} plots the SFO complexities for SGD, Momentum, and Adam versus batch size. For SGD, SFO complexity was minimum at $b^\star = 2^2$; for Momentum, it was minimum at $b^\star = 2^3$. 
This implies that SGD and Momentum could not use large batch sizes,
as shown in (\ref{cbs_s_m_a_2}).
For Adam, SFO complexity was minimum at the critical batch size $b^\star = 2^{11}$ that was close to estimation of critical batch 
$b_{\mathrm{A}}^\star > b_{\mathrm{A}}^*$ with $b_{\mathrm{A}}^* < 2^{10}$, as shown in (\ref{cbs_s_m_a_2}).
We also checked that the elapsed time for Adam monotonically decreased for $b \leq 2^{11}$ and that the elapsed time for critical batch size $b^\star = 2^{11}$ was the shortest. The elapsed time for $b \geq 2^{12}$ increased with the SFO complexity, as shown in Figure \ref{fig2}
(see Tables \ref{table:sgd}, \ref{table:momentum}, \ref{table:adam1}, and \ref{table:adam2} in Appendix \ref{appendix:1}).

\subsection{ResNet-18 on the MNIST dataset}
\begin{figure*}[htbp]
\begin{tabular}{cc}
\begin{minipage}[t]{0.45\hsize}
\centering
\includegraphics[width=1\textwidth]{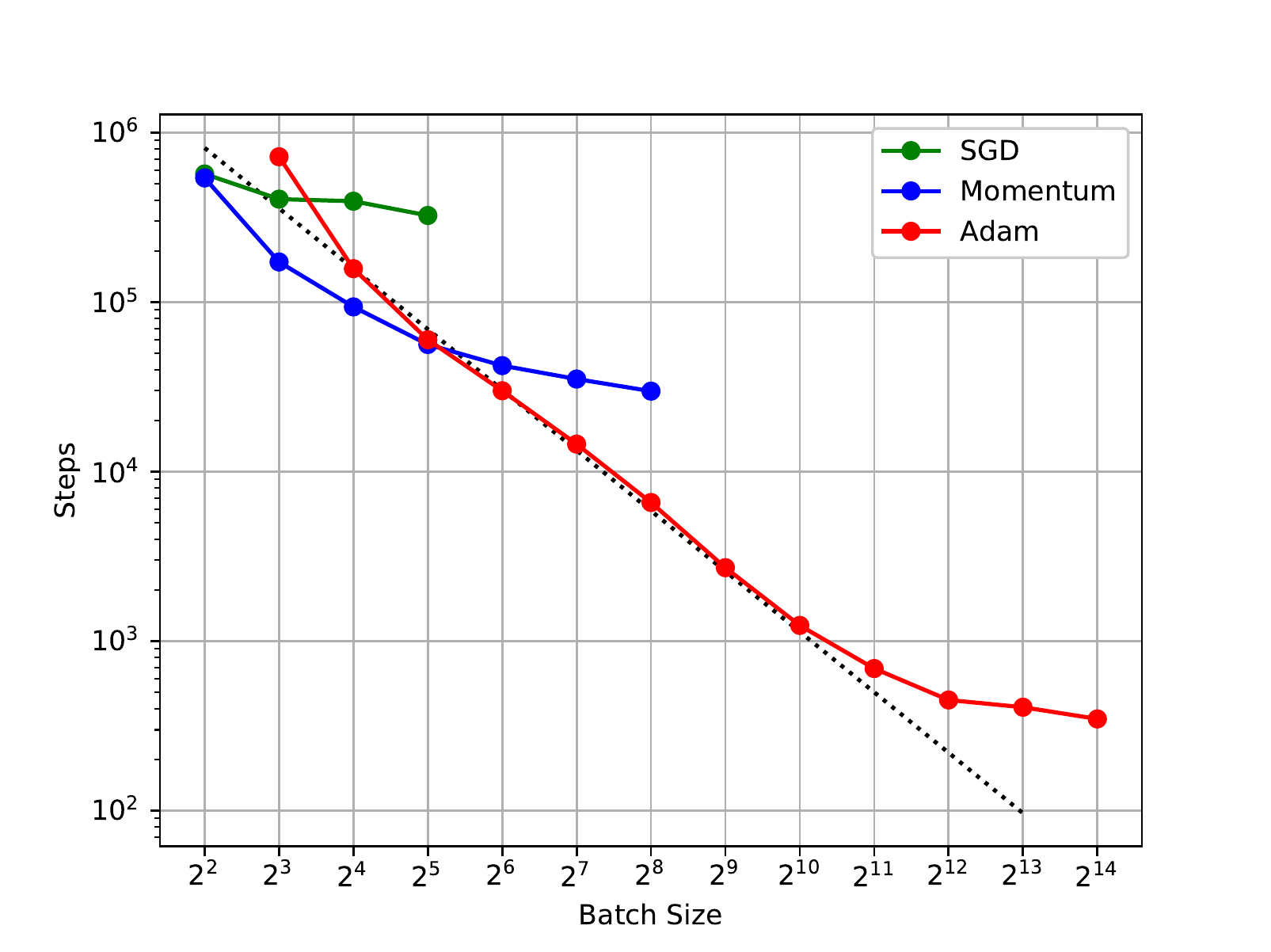}
\caption{Number of steps for SGD, Momentum, and Adam versus batch size needed to train ResNet-18 on MNIST. There is an initial period of perfect scaling (indicated by dashed line) such that the number of steps $K$ for Adam is inversely proportional to batch size $b$. Adam has critical batch size $b^\star = 2^{10}$.}
\label{fig3}
\end{minipage} &
\begin{minipage}[t]{0.45\hsize}
\centering
\includegraphics[width=1\textwidth]{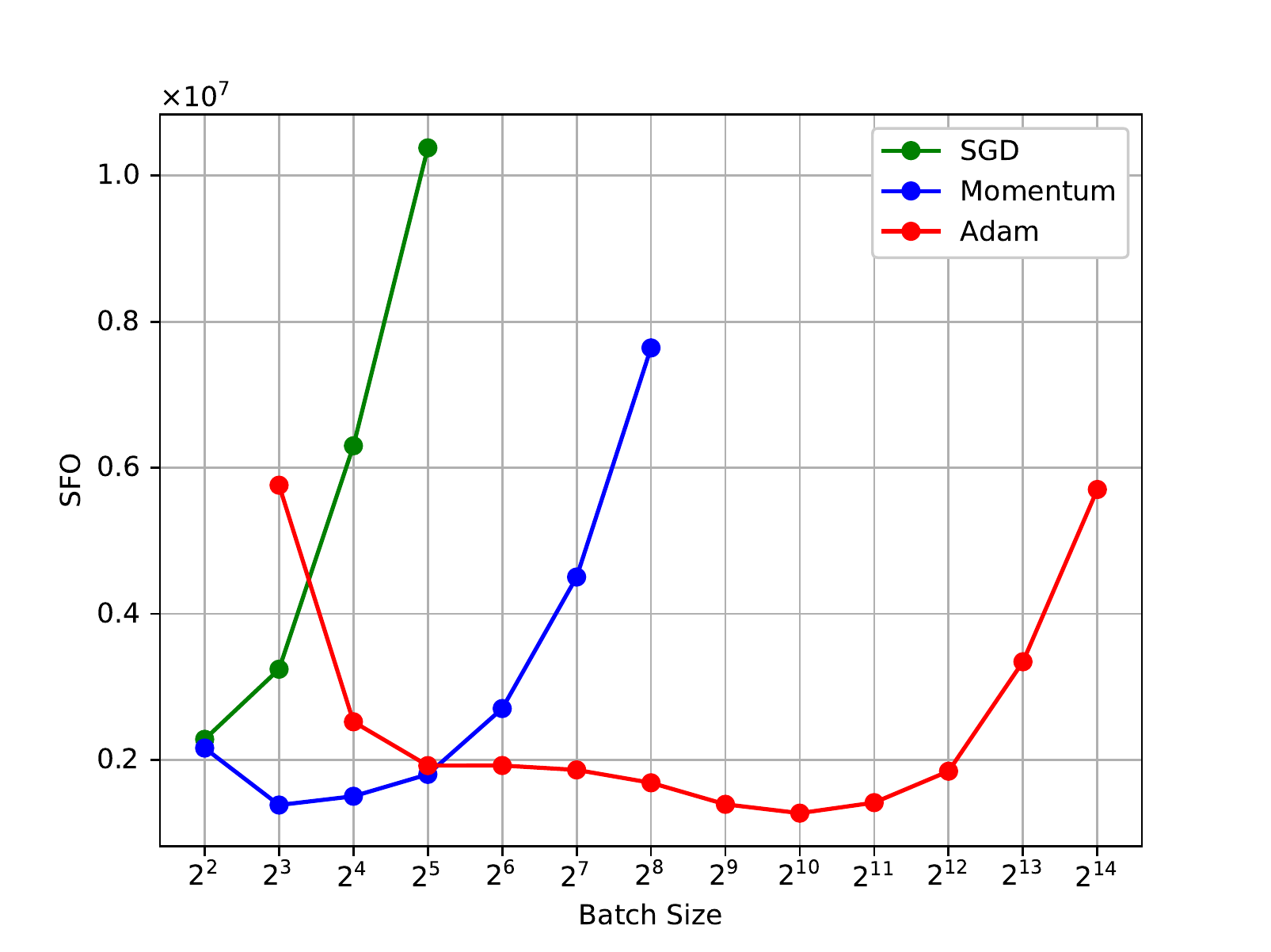}
\caption{SFO complexities for SGD, Momentum, and Adam versus batch size needed to train ResNet-18 on MNIST. 
SFO complexity of Adam (resp. Momentum) is minimized at critical batch size $b^\star = 2^{10}$ (resp. $b^\star = 2^3$), whereas SFO complexity for SGD tends to increase with batch size.}
\label{fig4}
\end{minipage}
\end{tabular}
\end{figure*}

Let us consider training ResNet-18 on the MNIST dataset with $n = 60000$.
Figures \ref{fig3} and \ref{fig4} indicate that Adam could exploit larger batch sizes than SGD and Momentum.
Moreover, these figures indicate that Momentum minimized SFO complexity at the critical batch size $b^\star = 2^3$ and
Adam minimized SFO complexity at the critical batch size $b^\star = 2^{10}$ that was close to estimation of critical batch 
$b_{\mathrm{A}}^\star > b_{\mathrm{A}}^*$ with $2^{11} < b_{\mathrm{A}}^* < 2^{12}$, as shown in (\ref{cbs_s_m_a_2}).
We can also check that the elapsed time for critical batch size $b^\star = 2^{10}$ was the shortest 
(see Tables \ref{table:sgd_m}, \ref{table:momentum_m}, \ref{table:adam1_m}, and \ref{table:adam2_m} in Appendix \ref{appendix:1}).


\section{Conclusion and Future Work}\label{sec:5}
This paper showed the relationship between batch size $b$ and the number of steps $K$ to achieve an $\epsilon$-approximation of deep learning optimizers using a small constant learning rate $\alpha$ and hyperparameters $\beta_1$ and $\beta_2$ close to $1$. 
From the convexity of SFO complexity $Kb$, there exists a global minimizer $b^\star$ of $Kb$ that is the critical batch size.  
We also gave numerical results indicating that Adam using a small constant learning rate, hyperparameters close to one, and the critical batch size minimizing SFO complexity has faster convergence than Momentum and SGD.
Moreover, we estimated appropriate batch sizes from our formula for $b^\star$ and showed that actual critical batch sizes are close to estimated batch sizes.

This paper focused on the theoretical analyses of SGD, Momentum, and Adam.
In the future, we should consider developing theoretical analyses of Adam's variants, such as Yogi, AMSGrad, AdaBelief, and AdamW.




\appendix
\section{Appendix}\label{appendix:1}
Unless stated otherwise, all relationships between random variables are
supposed to hold almost surely. Let $S$ be a positive definite matrix, which is denoted by $S \in \mathbb{S}_{++}^d$. The $S$-inner product of $\mathbb{R}^d$ is defined for all $\bm{x}, \bm{y} \in \mathbb{R}^d$ by $\langle \bm{x},\bm{y} \rangle_S := \langle \bm{x}, S \bm{y} \rangle = \bm{x}^\top (S \bm{y})$, and the $S$-norm is defined by $\|\bm{x}\|_S := \sqrt{\langle \bm{x}, S \bm{x} \rangle}$.

\subsection{Lemmas}
\begin{lem}\label{lem:1}
Suppose that (C1), (C2)(\ref{gradient}), and (C3) hold. Then, Adam satisfies the following: for all $k\in\mathbb{N}$ and all $\bm{\theta} \in \mathbb{R}^d$,
\begin{align*}
\mathbb{E}\left[ \left\| \bm{\theta}_{k+1} - \bm{\theta} \right\|_{\mathsf{H}_k}^2 \right]
&=
\mathbb{E}\left[\left\| \bm{\theta}_{k} - \bm{\theta} \right\|_{\mathsf{H}_k}^2 \right]
+ \alpha^2 \mathbb{E}\left[\left\| \bm{\mathsf{d}}_k \right\|_{\mathsf{H}_k}^2 \right]\\
&\quad + 2 \alpha \left\{
\frac{\beta_{1}}{\tilde{\beta}_{1k}} 
\mathbb{E}\left[ (\bm{\theta} - \bm{\theta}_k)^\top \bm{m}_{k-1} \right] 
+\frac{\hat{\beta}_{1}}{\tilde{\beta}_{1k}} 
\mathbb{E}\left[ (\bm{\theta} - \bm{\theta}_k)^\top \nabla f (\bm{\theta}_k) \right]
\right\},
\end{align*}
where $\bm{\mathsf{d}}_k := - \mathsf{H}_k^{-1} \hat{\bm{m}}_k$, $\hat{\beta}_{1} := 1 - \beta_{1}$, and $\tilde{\beta}_{1k} := 1 - \beta_{1}^{k+1}$.
\end{lem}

\begin{proof}
Let $\bm{\theta} \in \mathbb{R}^d$ and $k\in\mathbb{N}$. The definition $\bm{\theta}_{k+1} := \bm{\theta}_{k} + \alpha \bm{\mathsf{d}}_k$ implies that 
\begin{align*}
\| \bm{\theta}_{k+1} - \bm{\theta} \|_{\mathsf{H}_k}^2
= 
\| \bm{\theta}_{k} - \bm{\theta} \|_{\mathsf{H}_k}^2
+ 2 \alpha \langle \bm{\theta}_{k} - \bm{\theta}, \bm{\mathsf{d}}_k \rangle_{\mathsf{H}_k}
+ \alpha^2 \|\bm{\mathsf{d}}_k\|_{\mathsf{H}_k}^2.
\end{align*}
Moreover, the definitions of $\bm{\mathsf{d}}_k$, $\bm{m}_k$, and $\hat{\bm{m}}_k$ ensure that 
\begin{align*}
\left\langle \bm{\theta}_k - \bm{\theta}, \bm{\mathsf{d}}_k \right\rangle_{\mathsf{H}_k}
&=
\left\langle \bm{\theta}_k - \bm{\theta}, \mathsf{H}_k \bm{\mathsf{d}}_k \right\rangle
=
\left\langle \bm{\theta} - \bm{\theta}_k, \hat{\bm{m}}_k \right\rangle
=
\frac{1}{\tilde{\beta}_{1k}}
(\bm{\theta} - \bm{\theta}_k)^\top {\bm{m}}_k\\ 
&=
\frac{\beta_{1}}{\tilde{\beta}_{1k}} 
(\bm{\theta} - \bm{\theta}_k)^\top \bm{m}_{k-1} 
+
\frac{\hat{\beta}_{1}}{\tilde{\beta}_{1k}} 
(\bm{\theta} - \bm{\theta}_k)^\top \nabla f_{B_k}(\bm{\theta}_k).
\end{align*}
Hence, 
\begin{align}\label{ineq:004}
\begin{split}
\left\|\bm{\theta}_{k+1} - \bm{\theta} \right\|_{\mathsf{H}_k}^2
&=
\left\| \bm{\theta}_k -\bm{\theta} \right\|_{\mathsf{H}_k}^2
+ \alpha^2 \left\| \bm{\mathsf{d}}_k \right\|_{\mathsf{H}_k}^2\\
&\quad + 2 \alpha \left\{
\frac{\beta_{1}}{\tilde{\beta}_{1k}} 
(\bm{\theta} - \bm{\theta}_k)^\top \bm{m}_{k-1} 
+ \frac{\hat{\beta}_{1}}{\tilde{\beta}_{1k}} 
(\bm{\theta} - \bm{\theta}_k)^\top \nabla f_{B_k} (\bm{\theta}_k) 
\right\}.
\end{split}
\end{align}
Conditions (\ref{gradient}) and (C3) guarantee that
\begin{align*}
\mathbb{E}\left[ \mathbb{E} \left[(\bm{\theta} - \bm{\theta}_k)^\top \nabla f_{B_k} (\bm{\theta}_k) \Big| \bm{\theta}_k \right] \right]
=
\mathbb{E} \left[(\bm{\theta} - \bm{\theta}_k)^\top 
\mathbb{E} \left[\nabla f_{B_k} (\bm{\theta}_k) \Big| \bm{\theta}_k \right] \right]
=
\mathbb{E} \left[(\bm{\theta} - \bm{\theta}_k)^\top 
\nabla f (\bm{\theta}_k) \right].
\end{align*}
Therefore, the lemma follows from taking the expectation on both sides of (\ref{ineq:004}). This completes the proof.
\end{proof}

The discussion in the proof of Lemma \ref{lem:1} also gives the following lemma.

\begin{lem}\label{lem:1_sgd}
Suppose that (C1), (C2)(\ref{gradient}), and (C3) hold. 
Then, SGD satisfies the following: for all $k\in\mathbb{N}$ and all $\bm{\theta} \in \mathbb{R}^d$,
\begin{align*}
\mathbb{E}\left[ \left\| \bm{\theta}_{k+1} - \bm{\theta} \right\|^2 \right]
&=
\mathbb{E}\left[\left\| \bm{\theta}_{k} - \bm{\theta} \right\|^2 \right]
+ \alpha^2 \mathbb{E}\left[\left\| \nabla f_{B_k} (\bm{\theta}_k) \right\|^2 \right]
+ 2 \alpha  
\mathbb{E}\left[ (\bm{\theta} - \bm{\theta}_k)^\top \nabla f (\bm{\theta}_k) \right].
\end{align*}
Moreover, Momentum satisfies the following: for all $k\in\mathbb{N}$ and all $\bm{\theta} \in \mathbb{R}^d$,
\begin{align*}
\mathbb{E}\left[ \left\| \bm{\theta}_{k+1} - \bm{\theta} \right\|^2 \right]
&=
\mathbb{E}\left[\left\| \bm{\theta}_{k} - \bm{\theta} \right\|^2 \right]
+ \alpha^2 \mathbb{E}\left[\left\| \bm{m}_k \right\|^2 \right]\\
&\quad + 2 \alpha \left\{
\beta_{1} 
\mathbb{E}\left[ (\bm{\theta} - \bm{\theta}_k)^\top \bm{m}_{k-1} \right] 
+\hat{\beta}_{1} 
\mathbb{E}\left[ (\bm{\theta} - \bm{\theta}_k)^\top \nabla f (\bm{\theta}_k) \right]
\right\},
\end{align*}
where $\hat{\beta}_{1} := 1 - \beta_{1}$.
\end{lem}

The following lemma indicates the bounds on 
$(\mathbb{E}[ \|\bm{m}_k\|^2 ])_{k\in\mathbb{N}}$ and 
$(\mathbb{E}[ \|\bm{\mathsf{d}}_k \|_{\mathsf{H}_k}^2])_{k\in\mathbb{N}}$.

\begin{lem}\label{lem:bdd}
Adam under (C2)(\ref{gradient}), (\ref{sigma}), and (A1), for all $k\in\mathbb{N}$ satisfies
\begin{align*}
\mathbb{E}\left[ \left\|\bm{m}_k \right\|^2 \right] 
\leq 
\frac{\sigma^2}{b} + G^2, \quad 
\mathbb{E}\left[ \left\|\bm{\mathsf{d}}_k \right\|_{\mathsf{H}_k}^2 \right] 
\leq 
\frac{\sqrt{\tilde{\beta}_{2k}}}{\tilde{\beta}_{1k}^2 \sqrt{{v}_*}} \left( \frac{\sigma^2}{b} + G^2 \right),
\end{align*}
where ${v}_* := \inf \{ \min_{i\in [d]} {v}_{k,i} \colon k\in \mathbb{N}\}$, $\tilde{\beta}_{1k} := 1 - \beta_{1k}^{k+1}$, and $\tilde{\beta}_{2k} := 1 - \beta_{2k}^{k+1}$.
\end{lem}

\begin{proof}
Condition (C2)(\ref{gradient}) implies that
\begin{align}\label{equation_1}
\begin{split}
\mathbb{E} \left[ \left\| \nabla f_{B_k} (\bm{\theta}_{k}) \right\|^2
\Big| \bm{\theta}_k
\right]
&=
\mathbb{E} \left[\left\| \nabla f_{B_k} (\bm{\theta}_{k}) 
- \nabla f (\bm{\theta}_{k}) + \nabla f (\bm{\theta}_{k}) \right\|^2
\Big| \bm{\theta}_k
\right]\\
&=
\mathbb{E} \left[ \left\| \nabla f_{B_k} (\bm{\theta}_{k}) 
- \nabla f (\bm{\theta}_{k}) \right\|^2 \Big| \bm{\theta}_k
\right]
+ 
\mathbb{E} \left[ \left\| \nabla f (\bm{\theta}_{k}) \right\|^2 \Big| \bm{\theta}_k
\right]\\
&\quad + 2 
\mathbb{E} \left[ 
(\nabla f_{B_k} (\bm{\theta}_{k}) 
- \nabla f (\bm{\theta}_{k}))^\top \nabla f (\bm{\theta}_{k})
\Big| \bm{\theta}_k \right]\\
&= 
\mathbb{E} \left[\left\| \nabla f_{B_k} (\bm{\theta}_{k}) 
- \nabla f (\bm{\theta}_{k}) \right\|^2 \Big| \bm{\theta}_k
\right]
+ 
\| \nabla f (\bm{\theta}_{k}) \|^2,
\end{split}
\end{align} 
which, together with (C2)(\ref{sigma}) and (A1), in turn implies that 
\begin{align}\label{A3}
\mathbb{E} \left[ \left\| \nabla f_{B_k} (\bm{\theta}_{k}) \right\|^2
\right]
\leq 
\frac{\sigma^2}{b} + G^2.
\end{align}
The convexity of $\|\cdot\|^2$, together with the definition of $\bm{m}_k$ and (\ref{A3}), guarantees that, for all $k\in\mathbb{N}$,
\begin{align*}
\mathbb{E}\left[ \left\|\bm{m}_k \right\|^2 \right]
&\leq \beta_{1} \mathbb{E}\left[ \left\|\bm{m}_{k-1} \right\|^2 \right] + 
\hat{\beta}_{1} \mathbb{E}\left[ \left\|\nabla f_{B_k} (\bm{\theta}_k) \right\|^2 \right]\\
&\leq 
\beta_{1} \mathbb{E} \left[ \left\|\bm{m}_{k-1} \right\|^2 \right] + \hat{\beta}_{1} 
\left( \frac{\sigma^2}{b} + G^2 \right).
\end{align*}
Induction thus ensures that, for all $k\in\mathbb{N}$,
\begin{align}\label{induction}
\mathbb{E} \left[ \left\|\bm{m}_k \right\|^2 \right] \leq 
\max \left\{ \|\bm{m}_{-1}\|^2, \frac{\sigma^2}{b} + G^2 \right\} 
= \frac{\sigma^2}{b} + G^2,
\end{align}
where $\bm{m}_{-1} = \bm{0}$. For $k\in\mathbb{N}$, $\mathsf{H}_k \in \mathbb{S}_{++}^d$ guarantees the existence of a unique matrix $\overline{\mathsf{H}}_k \in \mathbb{S}_{++}^d$ such that $\mathsf{H}_k = \overline{\mathsf{H}}_k^2$ \cite[Theorem 7.2.6]{horn}. 
We have that, for all $\bm{x}\in\mathbb{R}^d$, $\|\bm{x}\|_{\mathsf{H}_k}^2 = \| \overline{\mathsf{H}}_k \bm{x} \|^2$. Accordingly, the definitions of $\bm{\mathsf{d}}_k$ and $\hat{\bm{m}}_k$ imply that, for all $k\in\mathbb{N}$, 
\begin{align*}
\mathbb{E} \left[ \left\| \bm{\mathsf{d}}_k \right\|_{\mathsf{H}_k}^2 \right]
= 
\mathbb{E} \left[ \left\| \overline{\mathsf{H}}_k^{-1} \mathsf{H}_k\bm{\mathsf{d}}_k \right\|^2 \right]
\leq 
\frac{1}{\tilde{\beta}_{1k}^2} \mathbb{E} \left[ \left\| \overline{\mathsf{H}}_k^{-1} \right\|^2 \|\bm{m}_k \|^2 \right],
\end{align*}
where 
\begin{align*}
\left\| \overline{\mathsf{H}}_k^{-1} \right\| 
= \left\| \mathsf{diag}\left(\hat{v}_{k,i}^{-\frac{1}{4}} \right) \right\| 
= \max_{i\in [d]} \hat{v}_{k,i}^{-\frac{1}{4}}
= \max_{i\in [d]} \left( \frac{v_{k,i}}{\tilde{\beta}_{2k}} \right)^{-\frac{1}{4}} 
=:
\left( \frac{v_{k,i^*}}{\tilde{\beta}_{2k}} \right)^{-\frac{1}{4}}. 
\end{align*} 
Moreover, the definition of 
\begin{align*}
{v}_* := \inf \left\{ {v}_{k,i^*} \colon k\in \mathbb{N} \right\}
\end{align*} 
and (\ref{induction}) imply that, for all $k\in \mathbb{N}$,
\begin{align*}
\mathbb{E} \left[ \left\| \bm{\mathsf{d}}_k \right\|_{\mathsf{H}_k}^2 \right] \leq 
\frac{\tilde{\beta}_{2k}^{\frac{1}{2}}}{\tilde{\beta}_{1k}^2 {v}_*^{\frac{1}{2}}} \left( \frac{\sigma^2}{b} + G^2 \right),
\end{align*}
completing the proof.
\end{proof}

\subsection{SGD}
We show Theorem \ref{thm:1} for SGD.

\begin{proof}
(i) Lemma \ref{lem:1_sgd} and (\ref{A3}) imply that, for all $k\in\mathbb{N}$ and all $\bm{\theta} \in \mathbb{R}^d$,
\begin{align*}
2 \alpha  
\mathbb{E}\left[ (\bm{\theta}_k - \bm{\theta})^\top \nabla f (\bm{\theta}_k) \right]
&=
\mathbb{E}\left[\left\| \bm{\theta}_{k} - \bm{\theta} \right\|^2 \right]
-
\mathbb{E}\left[ \left\| \bm{\theta}_{k+1} - \bm{\theta} \right\|^2 \right]
+ \alpha^2 \mathbb{E}\left[\left\| \nabla f_{B_k} (\bm{\theta}_k) \right\|^2 \right]\\
&\leq
\mathbb{E}\left[\left\| \bm{\theta}_{k} - \bm{\theta} \right\|^2 \right]
-
\mathbb{E}\left[ \left\| \bm{\theta}_{k+1} - \bm{\theta} \right\|^2 \right]
+ \alpha^2 \left(\frac{\sigma^2}{b} + G^2 \right).
\end{align*}
Summing the above inequality from $k=1$ to $k = K$ leads to the finding that, for all $K \geq 1$,
\begin{align*}
2 \alpha \sum_{k=1}^{K} 
\mathbb{E}\left[ (\bm{\theta}_k - \bm{\theta})^\top \nabla f (\bm{\theta}_k) \right]
\leq
\mathbb{E}\left[\left\| \bm{\theta}_{1} - \bm{\theta} \right\|^2 \right]
-
\mathbb{E}\left[ \left\| \bm{\theta}_{K+1} - \bm{\theta} \right\|^2 \right]
+ \alpha^2 \left(\frac{\sigma^2}{b} + G^2 \right) K,
\end{align*}
which implies that, for all $K \geq 1$ and all $\bm{\theta} \in \mathbb{R}^d$,
\begin{align}\label{vi_sgd}
\begin{split}
\mathrm{VI}(K,\bm{\theta}) 
&:= \frac{1}{K} \sum_{k=1}^{K} 
\mathbb{E}\left[ (\bm{\theta}_k - \bm{\theta})^\top \nabla f (\bm{\theta}_k) \right]
\leq
\frac{\mathbb{E}\left[\| \bm{\theta}_{1} - \bm{\theta}\|^2 \right]
}{2\alpha K}
+ 
\frac{\alpha}{2} \left(\frac{\sigma^2}{b} + G^2 \right)\\
&= 
\underbrace{\frac{\mathbb{E}\left[\| \bm{\theta}_{1} - \bm{\theta}\|^2 \right]
}{2\alpha}}_{C_1} \frac{1}{K}
+ 
\underbrace{\frac{\sigma^2 \alpha}{2}}_{C_2} \frac{1}{b}
+ 
\underbrace{\frac{G^2 \alpha}{2}}_{C_3}.
\end{split}
\end{align}

(ii)
Let $\bm{\theta} \in \mathbb{R}^d$ and $\epsilon > 0$.
Condition $C_1/K + C_2/b + C_3 = \epsilon$ is equivalent to  
\begin{align}\label{k_sgd}
K = K(b) = \frac{C_1 b}{(\epsilon - C_3)b - C_2}.
\end{align}
Since $\epsilon = C_1/K + C_2/b + C_3 > C_3$, 
we consider the case $b > C_2/(\epsilon - C_3) > 0$ to guarantee that $K > 0$.
From (\ref{vi_sgd}), the function $K$ defined by (\ref{k_sgd}) satisfies $\mathrm{VI}(K,\bm{\theta}) \leq C_1/K + C_2/b + C_3 = \epsilon$.
Moreover, from (\ref{k_sgd}),
\begin{align*}
\frac{\mathrm{d} K(b)}{\mathrm{d} b}
= 
\frac{- C_1 C_2}{\{ (\epsilon - C_3) b - C_2 \}^2} \leq 0, \text{ }
\frac{\mathrm{d}^2 K(b)}{\mathrm{d} b^2}
= 
\frac{2 C_1 C_2 (\epsilon - C_3)}{\{(\epsilon - C_3) b - C_2\}^3} \geq 0,
\end{align*}
which implies that $K$ is convex and monotone decreasing with respect to $b$.

(iii) 
We have that 
\begin{align*}
K b = K(b)b = \frac{C_1 b^2}{(\epsilon - C_3)b - C_2}.
\end{align*}
Accordingly, 
\begin{align*}
\frac{\mathrm{d} K(b) b}{\mathrm{d} b}
= 
\frac{C_1 b \{(\epsilon - C_3 )b - 2 C_2\}}
{\{(\epsilon - C_3 )b - C_2\}^2}, \text{ } 
\frac{\mathrm{d}^2 K(b) b}{\mathrm{d} b^2}
= 
\frac{2 C_1 C_2^2}{\{(\epsilon - C_3)b - C_2\}^3} \geq 0,
\end{align*}
which implies that $K(b) b$ is convex with respect to $b$ and 
\begin{align*}
\frac{\mathrm{d} K(b) b}{\mathrm{d} b}
\begin{cases}
< 0 &\text{ if } b < b^\star,\\
= 0 &\text{ if } b = b^\star = \frac{2 C_2}{\epsilon - C_3},\\
> 0 &\text{ if } b > b^\star.
\end{cases}
\end{align*}
The point $b^\star$ attains the minimum value 
$K(b^\star) b^\star = 2 C_1 C_2/(\epsilon - C_3)^2$ of $K(b) b$. This completes the proof.
\end{proof}

\subsection{Momentum}
We show Theorem \ref{thm:1} for Momentum.

\begin{proof}
(i) 
Lemma \ref{lem:1_sgd} ensures that, for all $k\in\mathbb{N}$ and all $\bm{\theta} \in \mathbb{R}^d$,
\begin{align*} 
\mathbb{E}\left[ (\bm{\theta}_k - \bm{\theta})^\top \bm{m}_{k-1} \right]
&=
\frac{1}{2 \alpha \beta_{1}}
\left\{
\mathbb{E}\left[\left\| \bm{\theta}_{k} - \bm{\theta} \right\|^2 \right]
-
\mathbb{E}\left[ \left\| \bm{\theta}_{k+1} - \bm{\theta} \right\|^2 \right]
\right\}
+ \frac{\alpha}{2  \beta_{1}} \mathbb{E}\left[\left\| \bm{m}_k \right\|^2 \right]\\
&\quad   
+ \frac{\hat{\beta}_{1}}{\beta_1} 
\mathbb{E}\left[ (\bm{\theta} - \bm{\theta}_k)^\top \nabla f (\bm{\theta}_k) \right],
\end{align*}
which, together with Lemma \ref{lem:bdd}, the Cauchy--Schwarz inequality,
and (A1) and (A2), implies that
\begin{align*}
\mathbb{E}\left[ (\bm{\theta}_k - \bm{\theta})^\top \bm{m}_{k-1} \right]
&\leq
\frac{1}{2 \alpha \beta_{1}}
\left\{
\mathbb{E}\left[\left\| \bm{\theta}_{k} - \bm{\theta} \right\|^2 \right]
-
\mathbb{E}\left[ \left\| \bm{\theta}_{k+1} - \bm{\theta} \right\|^2 \right]
\right\}
+ \frac{\alpha}{2  \beta_{1}} \left( \frac{\sigma^2}{b} + G^2  \right)\\
&\quad   
+\frac{\hat{\beta}_{1}}{\beta_1} \mathrm{Dist}(\bm{\theta}) G.
\end{align*} 
Summing the above inequality from $k=1$ to $k = K$ gives a relation that implies that 
\begin{align*}
\sum_{k=1}^{K} \mathbb{E}\left[ (\bm{\theta}_k - \bm{\theta})^\top \bm{m}_{k-1} \right]
&\leq
\frac{1}{2 \alpha \beta_{1}}
\left\{
\mathbb{E}\left[\left\| \bm{\theta}_{1} - \bm{\theta} \right\|^2 \right]
-
\mathbb{E}\left[ \left\| \bm{\theta}_{K+1} - \bm{\theta} \right\|^2 \right]
\right\}
+ \frac{\alpha}{2  \beta_{1}} \left( \frac{\sigma^2}{b} + G^2  \right)K\\
&\quad   
+\frac{\hat{\beta}_{1}}{\beta_1} \mathrm{Dist}(\bm{\theta}) G K,
\end{align*}
and hence, 
\begin{align*}
\frac{1}{K} \sum_{k=1}^{K} \mathbb{E}\left[ (\bm{\theta}_k - \bm{\theta})^\top \bm{m}_{k-1} \right]
\leq
\frac{\mathbb{E}\left[\| \bm{\theta}_{1} - \bm{\theta}\|^2 \right]}{2 \alpha \beta_{1} K}
+ \frac{\alpha}{2  \beta_{1}} \left( \frac{\sigma^2}{b} + G^2  \right)   
+\frac{\hat{\beta}_{1}}{\beta_1} 
\mathrm{Dist}(\bm{\theta}) G.
\end{align*}
Moreover,  
we have that, for all $k\in\mathbb{N}$ and all $\bm{\theta}\in \mathbb{R}^d$, 
\begin{align*}
(\bm{\theta}_k - \bm{\theta})^\top \bm{m}_{k}
&= 
(\bm{\theta}_k - \bm{\theta})^\top \bm{m}_{k-1} 
+ 
\hat{\beta}_1 (\bm{\theta}_k - \bm{\theta})^\top (\nabla f_{B_k}(\bm{\theta}_k) - \bm{m}_{k-1})\\
&\leq
(\bm{\theta}_k - \bm{\theta})^\top \bm{m}_{k-1} 
+ 
\hat{\beta}_1 \mathrm{Dist}(\bm{\theta}) \left( \|\nabla f_{B_k}(\bm{\theta}_k)\| + \| \bm{m}_{k-1}\|   \right),
\end{align*}
where the first equality comes from the definition of $\bm{m}_k$
and
the first inequality comes from 
the Cauchy--Schwarz inequality, the triangle inequality, and (A2).
Hence, from Lemma \ref{lem:bdd}, (\ref{A3}), Jensen's inequality,
and $b \geq 1$,
\begin{align}\label{M_k}
\begin{split}
\mathbb{E}\left[(\bm{\theta}_k - \bm{\theta})^\top \bm{m}_{k} \right]
&\leq
(\bm{\theta}_k - \bm{\theta})^\top \bm{m}_{k-1} 
+ 
2 \hat{\beta}_1 \mathrm{Dist}(\bm{\theta}) \sqrt{\frac{\sigma^2}{b} + G^2}\\
&\leq
(\bm{\theta}_k - \bm{\theta})^\top \bm{m}_{k-1} 
+ 
2 \hat{\beta}_1 \mathrm{Dist}(\bm{\theta}) \sqrt{\sigma^2 + G^2}.
\end{split}
\end{align}
Therefore, for all $K \geq 1$ and all $\bm{\theta}\in \mathbb{R}^d$, 
\begin{align}\label{m_k}
\begin{split}
&\frac{1}{K} \sum_{k=1}^{K} \mathbb{E}\left[ (\bm{\theta}_k - \bm{\theta})^\top \bm{m}_{k} \right]\\
&\leq
\frac{\mathbb{E}\left[\| \bm{\theta}_{1} - \bm{\theta}\|^2 \right]}{2 \alpha \beta_{1} K}
+
\frac{\sigma^2 \alpha}{2  \beta_{1} b}
+
\frac{G^2 \alpha}{2 \beta_1}
+
\hat{\beta}_{1}\mathrm{Dist}(\bm{\theta})
\left(
\frac{G}{\beta_1}
+
2\sqrt{\sigma^2 + G^2}
\right).
\end{split} 
\end{align}
The definition of $\bm{m}_k$ ensures that
\begin{align*}
&(\bm{\theta}_k - \bm{\theta})^\top \nabla f_{B_k}(\bm{\theta}_k)\\
&=
(\bm{\theta}_k - \bm{\theta})^\top \bm{m}_k
+ 
(\bm{\theta}_k - \bm{\theta})^\top (\nabla f_{B_k}(\bm{\theta}_k) - \bm{m}_{k-1}) 
+
(\bm{\theta}_k - \bm{\theta})^\top (\bm{m}_{k-1} - \bm{m}_{k})\\
&=
(\bm{\theta}_k - \bm{\theta})^\top \bm{m}_k
+ 
\frac{1}{\beta_{1}}(\bm{\theta}_k - \bm{\theta})^\top (\nabla f_{B_k}(\bm{\theta}_k) - \bm{m}_{k})
+
\hat{\beta}_{1} (\bm{\theta}_k - \bm{\theta})^\top (\bm{m}_{k-1} - \nabla f_{B_k}(\bm{\theta}_k)),
\end{align*}
which, together with the Cauchy--Schwarz inequality, the triangle inequality, and (A2), implies that
\begin{align*}
&(\bm{\theta}_k - \bm{\theta})^\top \nabla f_{B_k}(\bm{\theta}_k)\\
&\leq 
(\bm{\theta}_k - \bm{\theta})^\top \bm{m}_k
+ 
\frac{1}{\beta_{1}} \mathrm{Dist}(\bm{\theta}) (\|\nabla f_{B_k}(\bm{\theta}_k)\| + \|\bm{m}_{k}\|)
+ 
\hat{\beta}_{1} \mathrm{Dist}(\bm{\theta}) (\|\nabla f_{B_k}(\bm{\theta}_k)\| + \|\bm{m}_{k-1}\|).
\end{align*}
Lemma \ref{lem:bdd}, (\ref{A3}), Jensen's inequality, and $b \geq 1$ guarantee that
\begin{align}\label{M_k_1}
\begin{split}
\mathbb{E}\left[
(\bm{\theta}_k - \bm{\theta})^\top \nabla f(\bm{\theta}_k)
\right]
&\leq
\mathbb{E}\left[
(\bm{\theta}_k - \bm{\theta})^\top \bm{m}_k
\right]
+ 
2\left(\frac{1}{\beta_{1}} + \hat{\beta}_{1} \right)
\mathrm{Dist}(\bm{\theta}) \sqrt{\frac{\sigma^2}{b} + G^2}\\
&\leq
\mathbb{E}\left[
(\bm{\theta}_k - \bm{\theta})^\top \bm{m}_k
\right]
+ 
2\left(\frac{1}{\beta_{1}} + \hat{\beta}_{1} \right)
\mathrm{Dist}(\bm{\theta}) \sqrt{\sigma^2 + G^2}.
\end{split}
\end{align}
Therefore, (\ref{m_k}) ensures that, for all $K \geq 1$ and all $\bm{\theta} \in \mathbb{R}^d$, 
\begin{align*}
&\frac{1}{K} \sum_{k=1}^{K} \mathbb{E}\left[ (\bm{\theta}_k - \bm{\theta})^\top \nabla f(\bm{\theta}_k) \right]\\
&\leq
\underbrace{\frac{\mathbb{E}\left[\| \bm{\theta}_{1} - \bm{\theta}\|^2 \right]}{2 \alpha \beta_{1}}}_{C_1} \frac{1}{K}
+
\underbrace{\frac{\sigma^2 \alpha}{2  \beta_{1}}}_{C_2} \frac{1}{b} 
+
\underbrace{\frac{G^2 \alpha}{2\beta_1} 
+ \mathrm{Dist}(\bm{\theta})\left\{ \frac{G \hat{\beta}_1}{\beta_1} 
+ 2\sqrt{\sigma^2 + G^2}
\left(\frac{1}{\beta_1} + 2 \hat{\beta}_1 \right)  \right\}}_{C_3}.
\end{align*}

(ii)
A discussion similar to the one showing (ii) in Theorem \ref{thm:1} for SGD would show (ii) in Theorem \ref{thm:1} for Momentum.

(iii)
An argument similar to that which obtained (iii) in Theorem \ref{thm:1} for SGD would prove (iii) in Theorem \ref{thm:1} for Momentum.
\end{proof}

\subsection{Adam}
We show Theorem \ref{thm:1} for Adam.

\begin{proof}
(i) 
Let $\bm{\theta} \in \mathbb{R}^d$. Lemma \ref{lem:1} guarantees that for all $k\in \mathbb{N}$,
\begin{align}
\mathbb{E}\left[ (\bm{\theta}_k - \bm{\theta})^\top \bm{m}_{k-1} \right]
&= 
\underbrace{\frac{\tilde{\beta}_{1k}}{2 \alpha \beta_{1}}
\left\{
\mathbb{E}\left[\left\| \bm{\theta}_{k} - \bm{\theta} \right\|_{\mathsf{H}_k}^2 \right]
- 
\mathbb{E}\left[ \left\| \bm{\theta}_{k+1} - \bm{\theta} \right\|_{\mathsf{H}_k}^2 \right]
\right\}}_{a_k}
+ 
\underbrace{\frac{\alpha \tilde{\beta}_{1k}}{2 \beta_{1}} \mathbb{E}\left[\left\| \bm{\mathsf{d}}_k \right\|_{\mathsf{H}_k}^2 \right]}_{b_k} \nonumber \\
&\quad + 
\underbrace{\frac{\hat{\beta}_{1}}{\beta_{1}}
\mathbb{E}\left[ (\bm{\theta} - \bm{\theta}_k)^\top \nabla f (\bm{\theta}_k) \right]}_{c_k}.\label{key}
\end{align}
We define $\gamma_k := \tilde{\beta}_{1k}/(2 \beta_1 \alpha)$ $(k\in\mathbb{N})$.
Then, for all $K \geq 1$,
\begin{align}\label{LAM}
\begin{split}
\sum_{k = 1}^K a_k
&=
\gamma_1 \mathbb{E}\left[ \left\| \bm{\theta}_{1} - \bm{\theta} \right\|_{\mathsf{H}_{1}}^2\right]
+
\underbrace{
\sum_{k=2}^K \left\{
\gamma_k \mathbb{E}\left[ \left\| \bm{\theta}_{k} - \bm{\theta} \right\|_{\mathsf{H}_{k}}^2\right]
-
\gamma_{k-1} \mathbb{E}\left[ \left\| \bm{\theta}_{k} - \bm{\theta} \right\|_{\mathsf{H}_{k-1}}^2\right] 
\right\}
}_{{\Gamma}_K}\\
&\quad-
\gamma_{K} \mathbb{E} \left[ \left\| \bm{\theta}_{K+1} - \bm{\theta} \right\|_{\mathsf{H}_{K}}^2 \right].
\end{split}
\end{align}
Since $\overline{\mathsf{H}}_k \in \mathbb{S}_{++}^d$ exists such that $\mathsf{H}_k = \overline{\mathsf{H}}_k^2$, we have $\|\bm{x}\|_{\mathsf{H}_k}^2 = \| \overline{\mathsf{H}}_k \bm{x} \|^2$ for all $\bm{x}\in\mathbb{R}^d$. Accordingly, we also have 
\begin{align*}
{\Gamma}_K 
=
\mathbb{E} \left[ 
\sum_{k=2}^K 
\left\{
\gamma_{k} \left\| \overline{\mathsf{H}}_{k} (\bm{\theta}_{k} - \bm{\theta}) \right\|^2
-
\gamma_{k-1} \left\| \overline{\mathsf{H}}_{k-1} (\bm{\theta}_{k} - \bm{\theta}) \right\|^2
\right\}
\right].
\end{align*}
From $\overline{\mathsf{H}}_{k} = \mathsf{diag}(\hat{v}_{k,i}^{1/4})$, we have that, for all $\bm{x} = (x_i)_{i=1}^d \in \mathbb{R}^d$, $\| \overline{\mathsf{H}}_{k} \bm{x} \|^2 = \sum_{i=1}^d \sqrt{\hat{v}_{k,i}} x_i^2$. Hence, for all $K\geq 2$,
\begin{align}\label{DELTA}
{\Gamma}_K 
= 
\mathbb{E} \left[ 
\sum_{k=2}^K
\sum_{i=1}^d 
\left(
\gamma_{k} \sqrt{\hat{v}_{k,i}}
-
\gamma_{k-1} \sqrt{\hat{v}_{k-1,i}}
\right)
(\theta_{k,i} - \theta_i)^2
\right].
\end{align}
Condition (\ref{max_1}) and $\gamma_k \geq \gamma_{k-1}$ ($k \geq 1$) imply that, for all $k \geq 1$ and all $i\in [d]$,
\begin{align*}
\gamma_{k} \sqrt{\hat{v}_{k,i}} - \gamma_{k-1} \sqrt{\hat{v}_{k-1,i}} \geq 0.
\end{align*} 
Moreover, (A2) ensures that $D (\bm{\theta}) := \sup \{ \max_{i\in [d]} (\theta_{k,i} - \theta_i)^2 \colon k \in \mathbb{N} \} < + \infty$. Accordingly, for all $K \geq 2$,
\begin{align*}
{\Gamma}_K
\leq
D(\bm{\theta})
\mathbb{E} \left[ 
\sum_{k=2}^K
\sum_{i=1}^d 
\left(
\gamma_{k}\sqrt{\hat{v}_{k,i}} - \gamma_{k-1} \sqrt{\hat{v}_{k-1,i}}
\right)
\right]
= 
D(\bm{\theta})
\mathbb{E} \left[ 
\sum_{i=1}^d
\left(
\gamma_{K} \sqrt{\hat{v}_{K,i}}
-
\gamma_{1} \sqrt{\hat{v}_{1,i}}
\right)
\right].
\end{align*}
Let $\nabla f_{B_k}(\bm{\theta}_k) \odot \nabla f_{B_k}(\bm{\theta}_k) := (g_{k,i}^2) \in \mathbb{R}_{+}^d$. Assumption (A1) ensures that there exists $M \in \mathbb{R}$ such that, for all $k\in \mathbb{N}$, $\max_{i\in [d]} g_{k,i}^2 \leq M$. The definition of ${\bm{v}}_k$ guarantees that, for all $i\in [d]$ and all $k\in \mathbb{N}$,
\begin{align*}
v_{k,i} = \beta_{2} v_{k-1,i} + \hat{\beta}_{2} g_{k,i}^2.
\end{align*}
Induction thus ensures that, for all $i\in [d]$ and all $k\in \mathbb{N}$,
\begin{align}\label{v_M}
v_{k,i} \leq \max \{ v_{0,i}, M \} = M,
\end{align}
where $\bm{v}_0 = (v_{0,i}) = \bm{0}$. From the definition of $\hat{\bm{v}}_k$, we have that, for all $i\in [d]$ and all $k\in \mathbb{N}$,
\begin{align}\label{v_k}
\hat{v}_{k,i} = \frac{v_{k,i}}{\tilde{\beta}_{2k}}
\leq \frac{M}{\tilde{\beta}_{2k}}.
\end{align}
Therefore, (\ref{LAM}), $\mathbb{E} [\| \bm{\theta}_{1} - \bm{\theta}\|_{\mathsf{H}_{1}}^2] \leq D(\bm{\theta}) \mathbb{E} [ \sum_{i=1}^d \sqrt{\hat{v}_{1,i}}]$, and (\ref{v_k}) imply, for all $K\geq 1$,
\begin{align}\label{a_K_1_1}
\begin{split}
\sum_{k=1}^K a_k
&\leq
\gamma_{1} D(\bm{\theta}) \mathbb{E} \left[ 
\sum_{i=1}^d \sqrt{\hat{v}_{1,i}} \right]
+
D(\bm{\theta})
\mathbb{E} \left[
\sum_{i=1}^d 
\left(
\gamma_{K} \sqrt{\hat{v}_{K,i}}
-
\gamma_{1} \sqrt{\hat{v}_{1,i}}
\right)
\right]\\
&=
\gamma_{K} D(\bm{\theta})
\mathbb{E} \left[
\sum_{i=1}^d 
\sqrt{\hat{v}_{K,i}}
\right]\\
&\leq 
{\gamma}_K D(\bm{\theta}) 
\sum_{i=1}^d 
\sqrt{\frac{M}{\tilde{\beta}_{2K}}}\\
&\leq 
\frac{d D(\bm{\theta}) \sqrt{M} \tilde{\beta}_{1K}}{2 \beta_1 \alpha \sqrt{\tilde{\beta}_{2K}}}.
\end{split}
\end{align}
Inequality (\ref{a_K_1_1}) with $\tilde{\beta}_{1K} := 1 - \beta_1^{K+1}
\leq 1$ and $\tilde{\beta}_{2K} := 1 - \beta_2^{K+1} \geq 1 - \beta_2 =: \hat{\beta}_2$ implies that
\begin{align}\label{a_K_1}
\sum_{k=1}^K a_k
\leq
\frac{d D(\bm{\theta}) \sqrt{M} \tilde{\beta}_{1K}}{2 \beta_1 \alpha \sqrt{\tilde{\beta}_{2K}}}
\leq
\frac{d D(\bm{\theta}) \sqrt{M}}{2 \beta_1 \alpha \sqrt{\hat{\beta}_2}}.
\end{align}
Lemma \ref{lem:bdd} guarantees that, for all $k\in \mathbb{N}$, 
\begin{align}\label{B_k}
b_k = 
\frac{\alpha \tilde{\beta}_{1k}}{2 \beta_{1}} \mathbb{E}\left[\left\| \bm{\mathsf{d}}_k \right\|_{\mathsf{H}_k}^2 \right]
\leq
\frac{\alpha \tilde{\beta}_{1k}}{2 \beta_{1}}
\frac{\sqrt{\tilde{\beta}_{2k}}}{\tilde{\beta}_{1k}^2 \sqrt{{v}_*}} \left( \frac{\sigma^2}{b} + G^2 \right)
= 
\frac{\alpha \sqrt{\tilde{\beta}_{2k}}}{2 \sqrt{v_*} \beta_{1} \tilde{\beta}_{1k}} \left( \frac{\sigma^2}{b} + G^2 \right).
\end{align}
Inequality (\ref{B_k}) with $\tilde{\beta}_{1k} := 1 - \beta_1^{k+1} \geq 1 - \beta_1 =: \hat{\beta}_1$ and $\tilde{\beta}_{2k} := 1 - \beta_2^{k+1} \leq 1$ implies that 
\begin{align}\label{B_k_1}
b_k 
\leq
\frac{\alpha \sqrt{\tilde{\beta}_{2k}}}{2 \sqrt{v_*} \beta_{1} \tilde{\beta}_{1k}} \left( \frac{\sigma^2}{b} + G^2 \right)
\leq 
\frac{\alpha}{2 \sqrt{v_*} \beta_{1} \hat{\beta}_1} \left( \frac{\sigma^2}{b} + G^2 \right).
\end{align}
The Cauchy--Schwarz inequality and (A2) imply that, for all $k\in \mathbb{N}$,
\begin{align}\label{C_k_1}
c_k = 
\frac{\hat{\beta}_{1}}{\beta_{1}}
\mathbb{E}\left[ (\bm{\theta} - \bm{\theta}_k)^\top \nabla f (\bm{\theta}_k) \right]
\leq
\mathrm{Dist}(\bm{\theta}) G \frac{\hat{\beta}_{1}}{\beta_{1}}.
\end{align}
Hence, (\ref{key}), (\ref{a_K_1}), (\ref{B_k_1}), and (\ref{C_k_1}) ensure that, for all $K\geq 1$,
\begin{align*}
\frac{1}{K}\sum_{k=1}^K \mathbb{E}\left[ (\bm{\theta}_k - \bm{\theta})^\top \bm{m}_{k-1} \right]
\leq
\frac{d D(\bm{\theta}) \sqrt{M}}{2 \beta_1 \alpha \sqrt{\hat{\beta}_{2}}K}
+
\frac{ \alpha(\sigma^2 b^{-1} + G^2)}{2 \sqrt{v_*} \beta_{1} \hat{\beta}_1}  
+
\mathrm{Dist}(\bm{\theta}) G \frac{\hat{\beta}_{1}}{\beta_{1}}.
\end{align*}
Therefore, from (\ref{M_k}) and (\ref{M_k_1}), for all $K \geq 1$,
\begin{align*}
&\frac{1}{K}\sum_{k=1}^K \mathbb{E}\left[ (\bm{\theta}_k - \bm{\theta})^\top \nabla f(\bm{\theta}_{k}) \right]\\
&\leq
\underbrace{\frac{d D(\bm{\theta}) \sqrt{M}}{2 \alpha \beta_1 \sqrt{\hat{\beta}_2}}}_{C_1}
\frac{1}{K} 
+
\underbrace{\frac{\sigma^2 \alpha}{2 \sqrt{v_*} \beta_1 \hat{\beta}_1}}_{C_2} \frac{1}{b}
+ \underbrace{\frac{G^2 \alpha}{2 \sqrt{v_*} \beta_1 \hat{\beta}_1} 
+ \mathrm{Dist}(\bm{\theta})\left\{ \frac{G \hat{\beta}_1}{\beta_1} 
+ 2\sqrt{\sigma^2 + G^2}
\left(\frac{1}{\beta_1} + 2 \hat{\beta}_1 \right)  \right\}}_{C_3}.
\end{align*}

(ii)
A discussion similar to the one showing (ii) in Theorem \ref{thm:1} for SGD would show (ii) in Theorem \ref{thm:1} for Adam.

(iii)
An argument similar to that which obtained (iii) in Theorem \ref{thm:1} for SGD would prove (iii) in Theorem \ref{thm:1} for Adam.
\end{proof}

\newpage 

\subsection{Additional numerical results}
\begin{table*}[h]
\centering
\caption{Elapsed time and training accuracy of SGD when $f(\bm{\theta}_K) \leq 10^{-1}$ for training ResNet-20 on CIFAR-10}\label{table:sgd}
\begin{tabular}{lllllllll}
\bottomrule
\multicolumn{9}{c}{SGD} \\
Batch Size 
& $2^2$ 
& $2^3$ 
& $2^4$ 
& $2^5$ 
& $2^6$ 
& $2^7$ 
& $2^8$ 
& $2^9$ \\
\hline
Time (s) 
& 16983.64 
& 9103.76 
& 6176.19 
& 3759.25 
& --- 
& --- 
& --- 
& --- \\
Accuracy (\%) 
& 96.75 
& 96.69
& 96.66 
& 96.88 
& --- 
& --- 
& --- 
& --- \\
\bottomrule
\end{tabular}
\end{table*}

\begin{table*}[th]
\centering
\caption{Elapsed time and training accuracy of Momentum when $f(\bm{\theta}_K) \leq 10^{-1}$ for training ResNet-20 on CIFAR-10}\label{table:momentum}
\begin{tabular}{lllllllll}
\bottomrule
\multicolumn{9}{c}{Momentum} \\
Batch Size 
& $2^2$ 
& $2^3$ 
& $2^4$ 
& $2^5$ 
& $2^6$ 
& $2^7$ 
& $2^8$ 
& $2^9$ \\
\hline
Time (s) 
& 7978.90 
& 3837.72 
& 2520.82 
& 1458.70 
& 887.01 
& 678.66 
& 625.10 
& 866.65 \\
Accuracy (\%) 
& 96.49 
& 96.79
& 96.51 
& 96.72 
& 96.70 
& 96.94 
& 96.94 
& 98.34 \\
\bottomrule
\end{tabular}
\end{table*}

\begin{table*}[th]
\centering
\caption{Elapsed time and training accuracy of Adam when $f(\bm{\theta}_K) \leq 10^{-1}$ for training ResNet-20 on CIFAR-10}\label{table:adam1}
\begin{tabular}{lllllllll}
\bottomrule
\multicolumn{9}{c}{Adam} \\
Batch Size 
& $2^2$ 
& $2^3$ 
& $2^4$ 
& $2^5$ 
& $2^6$ 
& $2^7$ 
& $2^8$ 
& $2^9$ \\
\hline
Time (s) 
& 10601.78 
& 4405.73
& 2410.28
& 1314.01 
& 617.14 
& 487.75 
& 281.74 
& 225.03 \\
Accuracy (\%) 
& 96.46 
& 96.38
& 96.65 
& 96.53 
& 96.43 
& 96.68 
& 96.58 
& 96.72 \\
\bottomrule
\end{tabular}
\end{table*}

\begin{table*}[ht]
\centering
\caption{Elapsed time and training accuracy of Adam when $f(\bm{\theta}_K) \leq 10^{-1}$ for training ResNet-20 on CIFAR-10}\label{table:adam2}
\begin{tabular}{llllllll}
\bottomrule
\multicolumn{8}{c}{Adam} \\
Batch Size 
& $2^{10}$ 
& $\mathbf{2^{11}}$ 
& $2^{12}$ 
& $2^{13}$ 
& $2^{14}$ 
& $2^{15}$ 
& $2^{16}$ \\
\hline
Time (s) 
& 197.78 
& 195.40
& 233.70
& 349.81 
& 691.04 
& 644.19 
& 1148.68\\
Accuracy (\%) 
& 96.74 
& 97.21
& 97.54 
& 97.75 
& 97.51 
& 99.05 
& 99.03 \\
\bottomrule
\end{tabular}
\end{table*}

\begin{table*}[th]
\centering
\caption{Elapsed time and training accuracy of SGD when $f(\bm{\theta}_K) \leq 10^{-1}$ for training ResNet-18 on MNIST}\label{table:sgd_m}
\begin{tabular}{lllllllll}
\bottomrule
\multicolumn{9}{c}{SGD} \\
Batch Size 
& $2^2$ 
& $2^3$ 
& $2^4$ 
& $2^5$ 
& $2^6$ 
& $2^7$ 
& $2^8$ 
& $2^9$ \\
\hline
Time (s) 
& 1620.98 
& 1002.89 
& 1140.47 
& 1298.52 
& --- 
& --- 
& --- 
& --- \\
Accuracy (\%) 
& 99.70 
& 99.69
& 99.70 
& 99.70 
& --- 
& --- 
& --- 
& --- \\
\bottomrule
\end{tabular}
\end{table*}

\begin{table*}[ht]
\centering
\caption{Elapsed time and training accuracy of Momentum when $f(\bm{\theta}_K) \leq 10^{-1}$ for training ResNet-18 on MNIST}\label{table:momentum_m}
\begin{tabular}{lllllllll}
\bottomrule
\multicolumn{9}{c}{Momentum} \\
Batch Size 
& $2^2$ 
& $2^3$ 
& $2^4$ 
& $2^5$ 
& $2^6$ 
& $2^7$ 
& $2^8$ 
& $2^9$ \\
\hline
Time (s) 
& 1949.66 
& 550.28 
& 298.72 
& 238.66 
& 268.98 
& 362.89 
& 567.54 
& --- \\
Accuracy (\%) 
& 99.68 
& 99.65
& 99.68 
& 99.69 
& 99.71 
& 99.70 
& 99.72 
& --- \\
\bottomrule
\end{tabular}
\end{table*}

\begin{table*}[ht]
\centering
\caption{Elapsed time and training accuracy of Adam when $f(\bm{\theta}_K) \leq 10^{-1}$ for training ResNet-18 on MNIST}\label{table:adam1_m}
\begin{tabular}{lllllllll}
\bottomrule
\multicolumn{9}{c}{Adam} \\
Batch Size 
& $2^2$ 
& $2^3$ 
& $2^4$ 
& $2^5$ 
& $2^6$ 
& $2^7$ 
& $2^8$ 
& $2^9$ \\
\hline
Time (s) 
& --- 
& 7514.43
& 614.27
& 297.97 
& 216.10 
& 160.93 
& 127.97 
& 102.98 \\
Accuracy (\%) 
& --- 
& 99.74
& 99.66 
& 99.64 
& 99.66 
& 99.66 
& 99.69 
& 99.71 \\
\bottomrule
\end{tabular}
\end{table*}

\begin{table*}[ht]
\centering
\caption{Elapsed time and training accuracy of Adam when $f(\bm{\theta}_K) \leq 10^{-1}$ for training ResNet-18 on MNIST}\label{table:adam2_m}
\begin{tabular}{llllll}
\bottomrule
\multicolumn{6}{c}{Adam} \\
Batch Size 
& $\mathbf{2^{10}}$ 
& $2^{11}$ 
& $2^{12}$ 
& $2^{13}$ 
& $2^{14}$  \\
\hline
Time (s) 
& 93.53 
& 100.22
& 129.29
& 217.23 
& 375.87\\
Accuracy (\%) 
& 99.68 
& 99.67
& 99.67 
& 99.70 
& 99.69 \\
\bottomrule
\end{tabular}
\end{table*}

\end{document}